\documentclass{article}

\PassOptionsToPackage{numbers, compress}{natbib}

\usepackage[preprint]{neurips_2026}

\usepackage[utf8]{inputenc}
\usepackage[T1]{fontenc}
\usepackage{hyperref}
\usepackage{url}
\usepackage{booktabs}
\usepackage{amsfonts}
\usepackage{nicefrac}
\usepackage{microtype}
\usepackage{xcolor}

\usepackage{enumitem}
\usepackage{amsmath}
\usepackage{amssymb}
\usepackage{mathtools}
\usepackage{amsthm}
\usepackage{microtype}
\usepackage{graphicx}
\usepackage{subcaption}
\usepackage{multirow}
\usepackage{pifont}
\usepackage[table]{xcolor}

\usepackage{booktabs}
\usepackage{multirow}
\usepackage{makecell}
\usepackage{array}
\usepackage{graphicx}
\usepackage{wrapfig}

\theoremstyle{plain}
\newtheorem{theorem}{Theorem}[section]

\newtheorem{lemma}[theorem]{Lemma}
\newtheorem{corollary}[theorem]{Corollary}
\theoremstyle{definition}

\theoremstyle{remark}

\newcommand{\Cact}{\mathcal{W}^{\text{act}}}

\title{Efficient Serving for Dynamic Agent Workflows with Prediction-based KV-Cache Management}

\author{%
  \textbf{Haoyu Zheng}\textsuperscript{1}\quad
  \textbf{Fangcheng Fu}\textsuperscript{3}\quad
  \textbf{Jia Wu}\textsuperscript{4}\quad
  \textbf{Binhang Yuan}\textsuperscript{5}\quad
  \textbf{Yongqiang Zhang}\textsuperscript{2} \\
  \textbf{Hao Wang}\textsuperscript{1}\quad
  \textbf{Yuanyuan Zhu}\textsuperscript{1}\quad
  \textbf{Xiao Yan}\textsuperscript{1}\quad
  \textbf{Jiawei Jiang}\textsuperscript{1}\thanks{Corresponding author: \texttt{jiawei.jiang@whu.edu.cn}} \\[0.6em]
  \textsuperscript{1}\,Wuhan University \quad
  \textsuperscript{2}\,Dameng Database \quad
  \textsuperscript{3}\,Shanghai Jiao Tong University \\
  \textsuperscript{4}\,Macquarie University \quad
  \textsuperscript{5}\,HKUST
}

\newcommand{\stitle}[1]{\vspace{0em}\noindent{\bf #1\/}}
\newcommand{\substitle}[1]{\newline\noindent\textbf{#1}\hspace{0.5em}\ignorespaces}
\newcommand{\squishlist}{
	\begin{list}{$\bullet$}
		{ \setlength{\itemsep}{1pt}
			\setlength{\parsep}{1pt}
			\setlength{\topsep}{2.5pt}
			\setlength{\partopsep}{0.5pt}
			\setlength{\leftmargin}{1em}
			\setlength{\labelwidth}{1em}
			\setlength{\labelsep}{0.6em}
		}
	}
\newcommand{\squishend}{
	\end{list}
}
\newcommand{\squishenumatelist}{
  \begin{enumerate}[
    itemsep=1pt,
    parsep=1pt,
    topsep=2.5pt,
    partopsep=0.5pt,
    leftmargin=1.5em,
    labelwidth=1em,
    labelsep=0.6em
  ]
}
\newcommand{\squishenumatend}{
  \end{enumerate}
}

\begin{document}

\maketitle

\vspace{-20pt}
\begin{abstract}
\vspace{-5pt}

  LLM-based workflows compose specialized agents to execute complex tasks, and these agents usually share substantial context, allowing KV-Cache reuse to save computation. Existing approaches either manage KV-Cache at agent level and fail to exploit the reuse opportunities within workflows, or manage cache at the workflow level but assume that each workflow calls a static sequence of agents. However, practical workflows are typically dynamic, where the sequence of invoked agents and thus induced cache reuse opportunities depend on the context of each task. To serve such dynamic workflows efficiently, we build a system dubbed PBKV (\textbf{P}rediction-\textbf{B}ased \textbf{KV}-Cache Management). For each workflow, PBKV predicts the agent invocations in several future steps by fusing the guidance from historical workflows and context of the target workflow. Based on the predictions, PBKV estimates the reuse potential of cache entries and keeps the high-potential entries in GPU memory. To be robust to prediction errors, PBKV utilizes the predictions conservatively during both cache eviction and prefetching. Experiments on three workflow benchmarks show that PBKV achieves up to $1.85\times$ speedup over LRU on dynamic workflows, and up to $1.26\times$ speedup over the SOTA baseline KVFlow on the static workflow.

\end{abstract}
\vspace{-5pt}
\section{Introduction}
\vspace{-5pt}
\label{sec:Intro}

Large Language Model (LLM)-based multi-agent systems have become a popular paradigm for complex reasoning and automation.
Frameworks such as LangChain~\citep{langchain2022}, AutoGen~\citep{wu2024autogen}, and MetaGPT~\citep{hong2024metagpt} decompose a complex task into a group of collaborating \textit{agents} (e.g., a Planner, Coder, and Tester in code-generation pipelines), where each agent is invoked as a single LLM request.
Such multi-agent collaboration on a single task is known as an \emph{agentic workflow}.

Modern inference engines (e.g., vLLM~\citep{vllm-kwon2023efficient} and SGLang~\citep{sglang_NEURIPS2024_724be447}) share Key/Value tensors (i.e., KV-Cache) across requests, turning expensive prefill into cheap cache lookups.
In an agentic workflow, agents naturally share substantial context, including the system prompt, tool/agent descriptions, and the history accumulated by upstream agents~\citep{wu2026dualpathbreakingstoragebandwidth, langchain2022, langgraph2024}.
We term this reuse pattern \emph{workflow-level cache reuse}.
Ideally, the KV-Cache produced by one agent should be retained for the downstream agents, where the cache hit rate can exceed 90\%~\citep{wu2026dualpathbreakingstoragebandwidth}.
However, in practice, this potential is bounded by \textit{limited} GPU memory, making effective KV-Cache management crucial in multi-agent serving.

KV-Cache management is fundamentally a problem of \textit{predicting} future access patterns.
Existing inference engines commonly adopt the Least Recently Used (LRU) policy, which rests on a simple heuristic that the longest-idle cache is unlikely to be reused.
However, in multi-agent settings, cache reuse is determined by the \textit{workflow's structure} rather than temporal locality.
A cache entry may remain idle across multiple agent invocations (e.g., during tool calls or transitions between agents) before being reused, and LRU may evict it prematurely during this idle period.
As a result, LRU's temporal-locality assumption is misaligned with the structural nature of workflow-level cache reuse.

An idealized alternative is the classic Belady's algorithm~\citep{belady1966study}, which achieves \textit{offline}-optimal cache eviction by evicting the cache whose next access is farthest in the future.
The recent work KVFlow~\citep{pan2025kvflow} adapts this idea to multi-agent settings by assuming a \textit{predefined} and \textit{static} agent step graph (i.e., a global DAG that specifies the invocation orders of all agents), and evicting the KV-Cache of the agent whose next invocation is farthest away (i.e., with the largest `steps-to-execution').
This design rests on a strong assumption: the future invocation order of agents is \textit{known a priori}.
This assumption fails on realistic dynamic workloads, where workflows are \textit{runtime-dependent}.
For example, a Tester may send the code back for rewriting; and a Retriever may spawn new sub-queries based on intermediate results.
The realistic workflow structure cannot be determined in advance, thus effective KV-Cache management must \textbf{shift from assuming the future to predicting it}.

\stitle{Challenge.}
However, realizing this shift faces two conflicting challenges:
\textit{(i)} agentic workloads are inherently hard to predict~\citep{luo2025autellix,sui2026act_paste,wagenländer2026scepsyservingagenticworkflows_Scepsy,zhang2026agentic1,zhang2026agentic2}, because the stochastic decoding of LLMs injects uncertainty at each step, which propagates along the workflow path and yields high variability in the resulting agent invocation trace.
Industrial studies from Microsoft~\citep{barke2026agentrx_ms1} and IBM~\citep{moshkovich2025beyond_ibm1,moshkovich2025taming_ibm2} report that such uncertainty is common and consequential in agentic systems.
\textit{(ii)} Prediction errors are disproportionately costly.
The cost of mistakes is amplified by the workflow: a cache miss in multi-agent serving may force re-prefill of tens of thousands of tokens \textit{accumulated} along the entire workflow context.

Taken together, the structure of agentic workloads exposes substantial reuse potential, yet inevitably entails prediction errors and the steep cost of erroneous evictions. This motivates the central question of our work: \textit{in a dynamic multi-agent system, how can we design a KV-Cache management framework that (i) is robust to predictor quality and (ii) benefits continuously even from imperfect predictors?}

\stitle{Our Solution.}
Our key insights: (i) \textit{multi-step} prediction with \textit{complementary} signals can effectively replace the static assumption; and (ii) \textit{conservative} designs can yield \textit{stable} and \textit{robust} gains.

\vspace{-1pt}
Specifically, \textbf{for the predictor}, we propose two design principles:
(i) \textit{Complementary signal fusion}, which combines \textit{cross-workflow} agent transition patterns with \textit{per-request} semantics;
(ii) \textit{Multi-step horizon}, which is motivated by the fact that cache reuse distances in agentic workflows can span several invocations (e.g., across a retry loop), so a single-step predictor cannot distinguish cache that will be reused later from useless cache. Multi-step prediction can prevent this myopia \textit{without} assuming a static workflow.
\textbf{For the cache manager}, we likewise propose two designs: (i) \textit{Hierarchical eviction.} Guided by the observation that the
private cache of a terminated workflow has negligible reuse potential regardless of any prediction, we reclaim such \textit{retired} cache first and then rank the
remaining active cache by a lookahead reuse score
from the predictor, yielding \textit{stable} gains in the common case
and \textit{graceful} degradation under poor predictions. (ii) \textit{Conservative prefetching}, which carefully \textit{trades off} the cost and benefit of prefetching and
consumes \textit{only} otherwise-idle GPU space and PCIe bandwidth, so that
it does not backfire even under poor predictions.

\vspace{-1pt}
Together, these designs form a \textbf{P}rediction-\textbf{B}ased \textbf{KV}-Cache management framework, which we call PBKV. We evaluate PBKV on realistic dynamic workflows, where it reduces the average workflow latency by up to $1.85\times$ and improves the KV-Cache hit rate by up to $2.55\times$ over LRU. On static workflows, PBKV outperforms KVFlow on these metrics by up to $1.26\times$ and $1.39\times$, respectively.
Furthermore, PBKV's predictor is \textit{pluggable}; to ensure \textit{robustness} across predictors, we prove that the performance is \textit{Lipschitz-continuous} in prediction error, establishing \textit{graceful degradation}.

\stitle{Contributions.}
In summary, we make the following contributions:
{\setlength\topsep{0pt}\setlength\partopsep{1pt}
\begin{itemize}

    \item We formalize KV-Cache management in multi-agent serving as a problem of \textit{predicting} future access patterns, and provide design guidelines for predictors in this setting.

    \item We propose PBKV, guided by two design principles, i.e., \textit{hierarchical eviction} and \textit{conservative prefetching}, that together identify cache with low reuse value for reclamation while preserving high-value cache. We further provide a \textit{Lipschitz guarantee} that PBKV's performance degrades \textit{gracefully} even under prediction error.

    \item We implement and evaluate PBKV across diverse workloads and LLMs. PBKV consistently outperforms LRU and the SOTA baseline in both cache hit rate and end-to-end performance.

\end{itemize}
}
\vspace{-10pt}
\section{Preliminaries}
\vspace{-5pt}
\label{sec:background}

\stitle{LLM-based Multi-Agent Serving.}
A multi-agent application can be abstracted as a directed graph $G = (V, E)$, where each node in $V$ corresponds to an agent and each edge in $E$ denotes an admissible transition between agents. Since $G$ enumerates \textit{every possible} transition pattern, we refer to it as the \textbf{global call graph}. Figure~\ref{fig:global-call-graph} shows a representative example for a code-generation task. Notably, $G$ can contain loops, capturing retry loops common in realistic agentic workflows.
A \textbf{workflow} is a single execution instance on $G$, consisting of an ordered sequence of agent invocations $a_1, a_2, \ldots$, where each $a_i \in V$ and each consecutive pair $(a_i, a_{i+1}) \in E$.
Each \textbf{agent invocation} is an LLM request with a specific prompt, and thus benefits from KV-Cache reuse to reduce latency.

Agents typically share a large amount of KV-Cache, which we classify into two categories: (i) \emph{global cache}, shared across agent instances from \textit{different} workflows (e.g., the system prompt, tool/agent descriptions, and knowledge documents); and (ii) \emph{private cache}, produced by upstream agents of a particular workflow and reusable only by downstream agents of the \emph{same} workflow. Cross-workflow reuse of private cache is generically infeasible, because even under identical user prompts, the stochastic decoding of LLMs makes the workflows diverge within a few tokens.

\stitle{Radix Tree.}
PBKV is built on \textit{SGLang}~\cite{sglang_NEURIPS2024_724be447}, which organizes prefix-shared cache as a \textit{Radix Tree}, with each node holding a contiguous token segment reusable by all requests sharing that prefix.
\textit{HiCache}~\cite{sglang_hicache_2025} extends the cache into a two-tier hierarchy where cache evicted from GPU is retained in host memory and swapped back on a hit, avoiding full re-prefill at the cost of a PCIe transfer.

\vspace{-10pt}
\section{System Overview}
\vspace{-8pt}
\label{sec:overview}

\begin{figure}[t]
    \centering
    \begin{minipage}{0.287\textwidth}
        \centering
        \includegraphics[width=\linewidth]{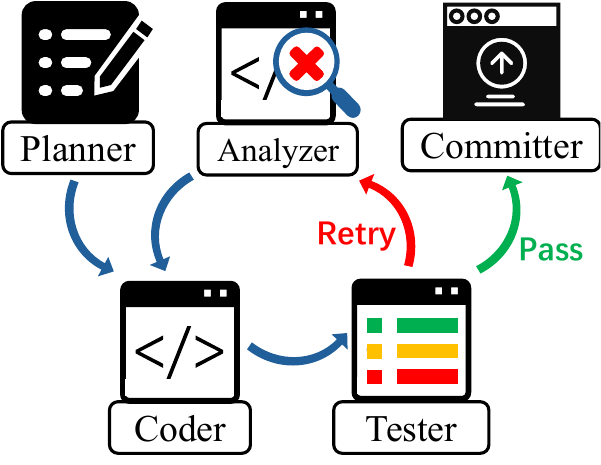}
\vspace{-13pt}
    \caption{A call graph for the code-generation task. The Tester conditionally triggers a retry path through Analyzer and Coder, i.e.,  a retry loop.}
    \label{fig:global-call-graph}
    \end{minipage}
    \hfill
    \begin{minipage}{0.648\textwidth}
        \centering
        \includegraphics[width=\linewidth]{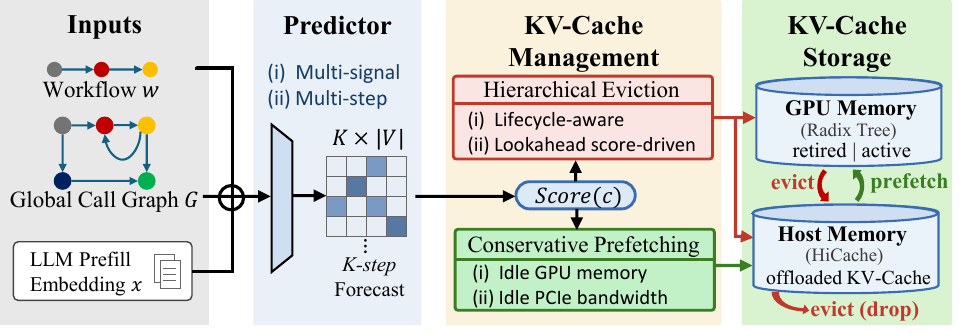}
\vspace{-13pt}
        \captionof{figure} {\textbf{System overview of PBKV.} For each active workflow $w$, the predictor produces a $K$-step forecast over upcoming agent invocations. The forecast drives a shared scoring function $Score(c)$, which feeds both a hierarchical eviction policy and a conservative prefetching policy on the two-tier KV-Cache storage.}
        \label{fig:overview}
    \end{minipage}
\vspace{-10pt}
\end{figure}

\vspace{-3pt}
The underlying design principle of PBKV is: predicting upcoming agent invocations to evaluate the \textit{reuse value} of existing KV-Cache, which then drives both KV-Cache eviction and prefetching. As shown in Figure~\ref{fig:overview}, PBKV consists of three components: (i) a \textit{predictor} that produces a multi-step forecast for each active workflow, (ii) a set of \textit{KV-Cache management policies} that translate the forecasts into eviction and prefetching decisions, and (iii) a \textit{two-tier KV-Cache storage} organized as a Radix Tree on GPU memory and HiCache on host memory, on which the policies operate.

\vspace{-5pt}
The predictor serves as the foundation of PBKV.
\textbf{First}, it fuses two complementary signals: (i) \textit{graph-level} agent transition patterns \textit{shared} across requests, encoded in the global call graph $G$;
and (ii) \textit{workflow-level} \textit{specifics} of the current request, reused from the LLM prefill embedding $x$.
\textbf{Second}, it emits $K$ probability distributions over upcoming agent invocations, rather than only the next one, preventing \textit{myopic eviction}. For instance, in Figure~\ref{fig:global-call-graph}, when the current agent is Analyzer, a single-step predictor would reveal only the next Coder invocation and miss the Tester re-invocation two steps later, risking a wrongful eviction of the Tester's cache and a costly re-prefill.

\vspace{-3pt}
The KV-Cache management policies consist of an eviction policy and a prefetching policy. They share a common reuse scoring mechanism (i.e., multi-step lookahead and cross-workflow aggregation) and a common design philosophy (i.e., embedding deterministic guardrails within a probabilistic system). Specifically, \textit{(i) hierarchical eviction} reclaims retired cache from terminated workflows first as it carries no reuse potential, and only after it is exhausted does score-driven eviction take over the active cache. Thus, performance degrades gracefully when the predictor is unreliable, while the upside is preserved when it is accurate. \textit{(ii) Conservative prefetching} is motivated by the asymmetry that a prefetch always pays its cost while the benefit materializes only when the prediction is correct. PBKV therefore restricts prefetching to otherwise-idle GPU space and PCIe bandwidth, so that even under poor predictions, prefetching neither displaces valuable cache nor competes for on-path bandwidth.

\vspace{-3pt}
In the following section, we elaborate on the design of each component and their coordination.

\vspace{-10pt}
\section{Design of PBKV}
\label{sec:design}
\vspace{-6pt}

In this part, we design a predictor as the foundation of subsequent KV-Cache Management.
\vspace{-3pt}

\subsection{Workflow Prediction}
\vspace{-5pt}
\label{sec:prediction}
\stitle{Predictor Design.} Building on the principles above (i.e., multi-signal fusion and multi-step horizon), we instantiate the predictor with two desiderata: (i) leveraging structural priors in the global call graph $G$, and (ii) generalizing to unseen workflow prefixes at runtime to accommodate the dynamic nature of agentic workflows. We adopt \textbf{GraphSAGE}~\citep{hamilton2017inductive_Graphsage} as the backbone, which satisfies both: (i) it operates directly on the \textit{graph} through neighborhood sampling and aggregation, preserving the structural priors, and (ii) it is \emph{inductive}, performing forward inference on unseen prefixes.

\vspace{-3pt}
\begin{wrapfigure}{r}{0.49\textwidth}
\vspace{-1.2em}
\centering
\includegraphics[width=\linewidth]{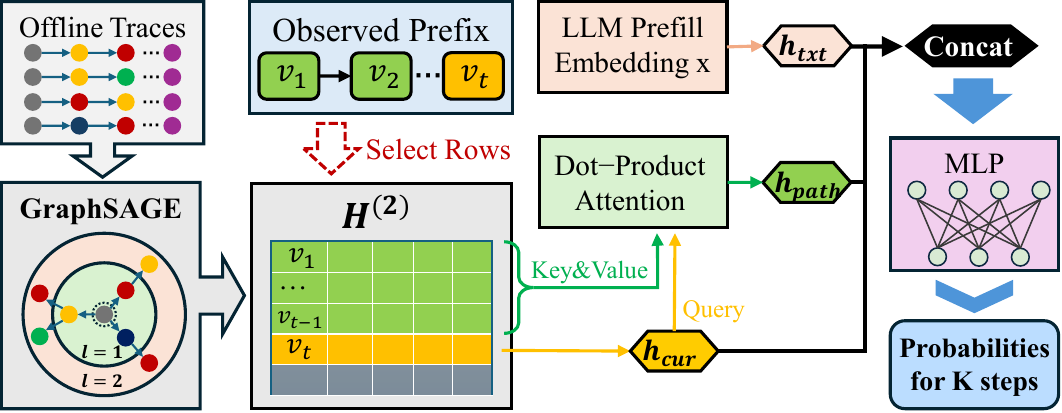}
\caption{\textbf{Architecture of the predictor.} It fuses a topology-aware agent embedding from GraphSAGE ($h_{cur}$), an attention-based workflow prefix summary ($h_{path}$), and a semantic signal reused from prefill ($h_{txt}$), then jointly predicts the next $K$ agent probability distributions via an MLP.}
\label{fig:predictor}
\vspace{-1.2em}
\end{wrapfigure}
As shown in Figure~\ref{fig:predictor},
the predictor fuses three complementary streams. \emph{(i) Topology-aware agent embedding} $\mathbf{h}_{\mathrm{cur}}$:
each agent's learnable embedding is refined by two GraphSAGE layers over a transition matrix estimated from offline traces, so that the resulting $\mathbf{H}^{(2)}$ encodes both the agent's identity and its neighborhood in $G$. \emph{(ii) Attention-based history aggregation} $\mathbf{h}_{\mathrm{path}}$: since two workflows at the same current agent may have arrived through very different prefixes, we summarize prior agent representations via scaled dot-product attention with $\mathbf{h}_{\mathrm{cur}}$ as the query, upweighting the most informative prefix agents over naive mean pooling. \emph{(iii) Semantic signal from prefill} $\mathbf{h}_{\mathrm{txt}}$:
to distinguish requests with the same prefix but diverging intent, we project the post-norm hidden state of the last prefill token, obtaining the signal essentially for free.
The three streams are concatenated and passed through a two-layer MLP that jointly emits logits over agents for each of the next $K$ steps in a single forward pass, avoiding autoregressive error accumulation while keeping per-invocation latency to that of a single inference.
Full architecture is detailed in Appendix~\ref{app:predictor-details}.

\vspace{-2pt}
\stitle{Training and Performance.} The predictor is trained on offline invocation traces with cross-entropy loss, pairing each prefix and its prefill embedding with the next $K$ agents as labels (padding positions after $\langle\mathrm{END}\rangle$ are masked). The resulting predictor has roughly 350K parameters and processes a batch of 1,024 requests in 1.56 ms, rendering its runtime overhead negligible. On the HoVer~\citep{jiang2020hover} dataset with the LangChain~\citep{langchain2022} framework, training on 1$K$ traces achieves an accuracy of 0.94 at 1-step and 0.77 at 3-step (500-trace test set). Scaling curves and comparisons against alternative predictors are provided in Appendix~\ref{app:predictor}, with full hyperparameters in the supplementary code.

\vspace{-5pt}
\subsection{Lookahead KV-Cache Eviction}
\vspace{-3pt}
\label{sec:eviction}

Building on the multi-step predictor, we design a \emph{lookahead} KV-Cache eviction policy.

\vspace{-5pt}
\subsubsection{Base: Lifecycle-Aware KV-Cache Eviction}
\vspace{-3pt}
\label{subsec:lifecycle}

\vspace{-1.5pt}
\stitle{Our observation and design.}
Once a workflow terminates, its residual private cache (which we call \textit{retired cache}) has negligible reuse potential and should be reclaimed first.
Under LRU, such cache can only age out passively, and LRU may evict still-valuable cache of other workflows during this aging window.
Additionally, retired cache is not all equivalent, we further rank them by the number of workflows that have accessed them. Thus, the popular shared prefix can be preserved longer.

\vspace{-1.5pt}

\stitle{Implementation.}
We tag each cache node with the workflows that have accessed it. The server continuously listens for and records workflow termination messages from clients. Any cache node whose associated workflows have \emph{all} terminated is tagged as retired and prioritized for eviction. Despite its simplicity, this change alone improves the average hit rate by up to $1.66\times$ in our experiments.

\vspace{-5pt}
\subsubsection{Lookahead Score-Driven KV-Cache Eviction}
\vspace{-3pt}
\label{subsec:score}
Retired cache is not always abundant. Under high load, once the retired cache is drained, lifecycle-aware eviction degenerates back to LRU.
We notice that the reuse likelihood also \textit{varies} across the active cache. For example, (i) among the \textit{global cache}, descriptions of popular agents are reused more frequently than those of rarely used ones; and (ii) among the \textit{private cache}, retention priority should scale with the predicted probability that the owning agent is invoked in the future.
We therefore extend eviction from a binary lifecycle label to a continuous score that reflects predicted reuse value.

\vspace{-2pt}
\stitle{Cross-Workflow Value Aggregation.}
KV-Cache is shared across workflows, thus its reuse value should be aggregated \textit{globally}.
Therefore, the score of a cache node should grow with \textit{(i)} the number of workflows likely to reuse it, and \textit{(ii)}  each workflow's reuse probability.
We further annotate each cache node $c$ with a per-workflow access indicator vector $A_w(c)\in \{0,1\}^{|V|}$.
For example, if $c$ has been accessed by Agents~1 and~3 from workflow
$w$, then $A_w(c)=[1,0,1,\dots]$.
For each workflow $w$, the predictor emits a next-step distribution, which we decompose into an agent-access probability vector $P_w$ and a termination probability $p_{w,\langle\text{END}\rangle}$. The single-step reuse value of $c$ is then defined as

\vspace{-10pt}

\begin{equation}
\mathrm{Value}(c) \;=\; \sum_{w \,\in\, \Cact(c)} A_w(c) \cdot P_w
\label{eq:single-step-value}
\end{equation}

\vspace{-10pt}

where $\Cact(c)$ denotes the set of active workflows associated with node $c$.
Intuitively, $\mathrm{Value}(c)$ \textit{aggregates} the probability that each workflow's next invocation will touch node $c$.
Notably, this cross-workflow aggregation \emph{naturally} protects the global cache and popular-prefix cache.

\vspace{-2pt}
\begin{wrapfigure}{r}{0.46\textwidth}
\vspace{-1.2em}
\centering
\includegraphics[width=\linewidth]{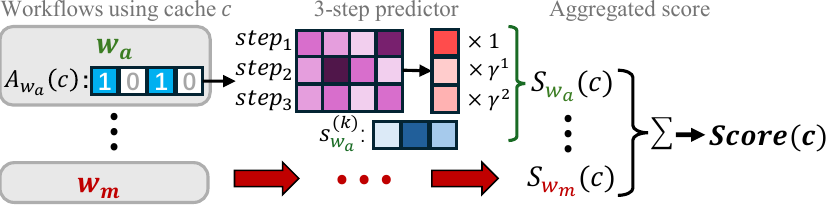}
    \caption{\textbf{Computing the cross-workflow reuse score.} For each active workflow $w$ accessing cache node $c$, the $K$-step ($K$=3 here) predictor outputs per-step access probabilities, which are weighted by the survival probability $s^{(k)}$ and confidence factor $\gamma^{k-1}$ and summed across $m$ workflows as $Score(c)$.}
\label{fig:score}
\vspace{-1.2em}
\end{wrapfigure}
\stitle{From One-Step to $K$-Step Lookahead.} As described above, a single-step view remains myopic, so we extend the score to a $K$-step horizon by leveraging the predictor's multi-step outputs.
Two practical effects must be accounted for: (i) predictions further into the future are objectively less reliable, regardless of the predictor's stated confidence; and (ii) steps following a predicted termination are meaningless and would pollute the score if left in.
To address them, we introduce a confidence decay factor $\gamma < 1$ and a \emph{cumulative survival probability}
$
s_w^{(k)} \;=\; \prod_{j=1}^{k-1} \bigl(1 - p_{w,\langle\text{END}\rangle}^{(j)}\bigr),
$
i.e., the probability that workflow $w$ remains active at step $k$.
Figure~\ref{fig:score} illustrates the computation, which we formalize as:

\vspace{-12pt}
\begin{equation}
    \mathrm{Score}(c)  \;=\; \sum_{k=1}^{K} \gamma^{k-1} \sum_{w \,\in\, \Cact(c)} s_w^{(k)} \cdot A_w(c) \cdot P_w^{(k)}
    \label{eq:multistep-score}
\end{equation}

\vspace{-6pt}

\stitle{Implementation.} On every workflow state change (i.e., new agent invocation or termination), the predictor refreshes its forecast and updates the scores of the affected cache nodes. On eviction, nodes are sorted in ascending order of $\mathrm{Score}(c)$ and evicted until the requested space is freed.

\vspace{-4pt}
\subsubsection{Hierarchical Eviction: Unifying the Two Policies}
\vspace{-2pt}
\label{subsec:hierarchical}

The two policies can be naturally combined because retired cache nodes receive a score of $0$ under Equation~\ref{eq:multistep-score} (as they satisfy $\Cact(c) = \emptyset$).
That said, the ``no value'' judgment of retired cache is \emph{deterministic}, whereas the ``zero score'' of an active node is merely a \emph{probabilistic} estimate bounded by both the horizon $K$ and the predictor's accuracy.
We therefore adopt a \emph{hierarchical eviction} strategy, in which active cache is spared until all retired cache is drained.
The underlying design philosophy is to \emph{embed deterministic guardrails within a probabilistic system}, so that performance degrades gracefully under unreliable predictions while preserving the upside under good predictions.

\vspace{-4pt}
\subsection{Conservative KV-Cache Prefetching}
\vspace{-4pt}
\label{sec:prefetch}

Building on the predictor, a complementary optimization is to proactively load likely-to-be-reused cache nodes into GPU memory before they are hit, leveraging SGLang's HiCache infrastructure to hide transfer latency behind ongoing decode steps rather than exposing it on the request critical path.

\vspace{-2pt}
\stitle{Prefetching Principle.}
At its core, prefetching trades \textit{known-valuable} GPU cache for \textit{speculatively valuable} host cache, with the \textit{fixed costs} in scheduling and PCIe bandwidth.
Given the inevitable prediction errors in dynamic workflows, the asymmetry between deterministic cost and probabilistic benefit motivates a \textit{conservative} principle that prefers risk avoidance over aggressive speculation.

\vspace{-2pt}
\stitle{Prefetching Design.}
Concretely, we do \textit{not} evict any active cache to make room for prefetched data, for three considerations: (i) as argued above, evicting known-valuable cache is inherently risky under dynamic workflows; (ii) under high concurrency, such evictions can disrupt the radix tree's prefix structure; and (iii) if the evicted active cache is hit shortly after, the system pays an additional GPU-reload cost or, in the worst case, a full re-prefill.
\textit{Instead}, we restrict the prefetch region to the union of currently \textit{free} space and \textit{retired} cache (§4.2.1), whose total size we denote as $S_a$.
This design guarantees that prefetching perturbs existing GPU-resident cache only minimally, if at all.

\vspace{-2pt}
Beyond GPU space, prefetching also consumes PCIe bandwidth and therefore \textit{competes} with prefill requests. To avoid interference, PBKV activates prefetching only on pure-decode batches (>90\% of all batches in our measurements).
Within such a batch, we bound the transferable volume as
$S_{bw}=Bandwidth \cdot StepDuration$,
where the PCIe bandwidth and decode step duration are dynamically updated by their runtime averages.
The final budget is $S=\min\{S_a,\ S_{bw}\}$, which typically hides prefetching cost behind a single decode step and uses only otherwise-idle bandwidth.

\vspace{-2pt}
To rank candidate nodes in the host memory, we reuse the one-step reuse value $Value(c)$ defined in Eq.~\ref{eq:single-step-value}.
We deliberately restrict ranking to the \textit{one-step} horizon because prefetching is naturally incremental: a node that will be reused two steps ahead can simply be fetched at the next step. In contrast, eviction must look ahead across $K$ steps, since a mistaken eviction incurs a costly reload or even a full re-prefill. Selecting the optimal subset under budget $S$ is formally a knapsack problem and hence \textit{NP-hard}. Since prefetch decisions lie on the scheduler's critical path, we use a lightweight greedy heuristic that loads candidate nodes in descending order of value until the budget is exhausted.

An aggressive variant that permits evicting active cache for prefetch space is evaluated in Appendix~\ref{app:aggressive-prefetch}.

\vspace{-8pt}
\section{Discussion}
\vspace{-5pt}
\label{sec:discussion}

\vspace{-3pt}
\stitle{Positioning.} While our primary contribution is KV-cache management, we additionally provide design guidelines for predictors in this setting (Section~\ref{sec:related_work} explains why existing predictors fall short) and instantiate one accordingly. PBKV deliberately keeps the predictor module \textit{pluggable}, so that it can benefit from future predictors. We further prove in Appendix~\ref{app:smoothness} that PBKV degrades \textit{gracefully} even under poor predictions (i.e., the degradation is \textit{Lipschitz-continuous}), that is:

\vspace{-5pt}

\begin{theorem}
\label{thm:smoothness}
Let $\widehat{E}_B$ and $E_B^{\star}$ be PBKV's eviction set and the cost-minimizing set under the ground truth, respectively, and $\epsilon_n^{\gamma}$ the per-node prediction error. The eviction cost regret of PBKV satisfies

\vspace{-13pt}

\begin{equation*}
0 \;\leq\; \mathcal{R}(B) \;:=\; \mathcal{L}\!\bigl(\widehat{E}_B\bigr) - \mathcal{L}\!\bigl(E_B^{\star}\bigr) \;\leq\; \frac{1}{2(1-\gamma)} \!\!\sum_{c\,\in\,\widehat{E}_B \triangle E_B^{\star}} \!\epsilon_c^{\gamma},
\end{equation*}

\vspace{-11pt}

where $\widehat{E}_B \triangle E_B^{\star}$ denotes the symmetric difference between PBKV's eviction set and the optimum. The bound $\mathcal{R}(B) \to 0$ when the prediction is perfect and grows linearly in the prediction error.
\end{theorem}

\vspace{-5pt}

\stitle{Complexity.}
PBKV maintains the score of candidate nodes in a \textit{heap} structure.
Inspecting the top candidate takes $O(1)$, while popping, insertion, and score updates take $O(\log n)$ for $n$ cache nodes, keeping the scheduler overhead negligible even under high concurrency (measured in Appendix~\ref{app:scheduler-overhead}).

\vspace{-2pt}
\stitle{Limitations.}
(i) The predictor needs to be trained on a \textit{specific} workload.
(ii) Cache-reuse patterns across agents are not universal but rather depend on the \textit{client's specific} message-passing convention. The indicator $A_w(c)$ in Eq.~\ref{eq:single-step-value} captures a \textit{basic form}, namely the set of agents that have directly accessed $c$ in workflow $w$. It admits straightforward \textit{customization}, e.g., when a framework specifies that agent $j$ inherits cache accessed by agent $i$, $A_w(c)[j]$ can be set to 1 whenever $A_w(c)[i]=1$.

\vspace{-5pt}
\section{Experimental Evaluation}
\vspace{-10pt}
\label{sec:exp}

In this section, we evaluate PBKV on realistic multi-agent workloads and ablate each component.

\vspace{-5pt}
\subsection{Experimental Settings}
\label{sec:exp-setup}

\vspace{-3pt}
\stitle{Testbed and Models.}
Our experiments are conducted on a server with $8\times$ NVIDIA A6000 (48\,GB) GPUs interconnected via NVLink, 128 virtual CPU cores, 512\,GB of memory, and 20\,GB/s PCIe bandwidth. We use Qwen3-14B and Qwen3-32B (two-way tensor parallelism) as the base LLMs.

\vspace{-3pt}
\stitle{Baselines.}
We compare PBKV against
\textbf{(i) LRU} on SGLang+HiCache (\texttt{HICACHE\_RATIO}=$1$, allocating equal capacity on GPU and host memory), the default policy of SGLang.
\textbf{(ii) KVFlow}~\cite{pan2025kvflow}, the SOTA workflow-aware policy, whose eviction is \textit{driven} by the `steps-to-execution' distance on a static DAG. Since this distance is \textit{undefined} under runtime-dependent \textit{loops}, we construct a static workflow scenario for fair comparison.
To isolate the contribution of each component, we further evaluate two ablation variants of PBKV: \textbf{(i) PBKV-LAE}, which enables only Lifecycle-Aware Eviction; and \textbf{(ii) PBKV-HE}, which adds Hierarchical Eviction on top of PBKV-LAE but disables prefetching.

\vspace{-3pt}
\stitle{Workloads.}
We evaluate PBKV on three representative multi-agent workloads, each pairing a public benchmark with a widely-adopted agent framework to reflect \textit{real-world} deployment:
\textit{(i) fact verification}, using the HoVer~\citep{jiang2020hover} dataset with the LangChain~\citep{langchain2022} agent framework;
\textit{(ii) code generation}, using SWE-bench~\citep{jimenez2024swebench} with Microsoft AutoGen~\citep{wu2024autogen};
and \textit{(iii) document analysis}, using FinanceBench~\citep{islam2023financebench} with CrewAI~\citep{crewai2023}, which constitutes a \textit{static} workflow for direct comparison against KVFlow.
Following prior work~\citep{pan2025kvflow,chen2026concur}, we set the number of concurrent workflows to induce memory pressure, thereby effectively evaluating KV-Cache management policies.

\vspace{-3pt}
\stitle{Metrics.}
We evaluate each policy along three \textit{complementary} dimensions:
\textit{(i) End-to-end workflow latency}, measured from workflow submission to final completion;
\textit{(ii) Per-agent TTFT}, defined as the mean time-to-first-token across every agent invocation along a workflow;
\textit{(iii) Average KV-Cache hit rate} (token-level) on GPU memory, which \textit{directly} reflects the quality of cache management.

\vspace{-8pt}
\subsection{Main Results.}
\vspace{-2pt}
\label{subsec:main-results}

\begin{table}[t]
\centering
\caption{Performance of policies across workloads and LLMs. The metrics are described above. As different workloads have different GPU memory footprints, we set the concurrency limit to $72$/$24$/$48$, respectively, and study its sensitivity in Section~\ref{sec:sensitivity}. The \textit{static} workload is designed for comparison with \smash{\colorbox{orange!15}{KVFlow}}, which assumes a predefined static call graph. (Full results with std see Appendix~\ref{app:full-results})}

\label{tab:main}
\resizebox{\linewidth}{!}{
\setlength{\tabcolsep}{2.5pt}
\begin{tabular}{ll@{\hskip 0pt}cccccc}
\toprule
\multirow{2}{*}[-0.5ex]{\textbf{Workload}} & \multirow{2}{*}[-0.5ex]{\textbf{Policy}} & \multicolumn{3}{c}{\textbf{Qwen3-32B}} & \multicolumn{3}{c}{\textbf{Qwen3-14B}} \\
\cmidrule(lr){3-5} \cmidrule(lr){6-8}
& & \textbf{Lat. (s)$\downarrow$} & \textbf{PA. TTFT (s)$\downarrow$} & \textbf{Hit Rate (\%)$\uparrow$} & \textbf{Lat. (s)$\downarrow$} & \textbf{PA. TTFT (s)$\downarrow$} & \textbf{Hit Rate (\%)$\uparrow$} \\
\midrule
\multirow{4}{*}{\shortstack[l]{\textbf{HoVer}\\\textbf{+ LangChain}\\ (with iterative\\ refinement \textit{loops})}}
& LRU       & $189.66$ & $16.65$ & $27.09$ & $139.59$ & $10.90$ & $28.79$ \\
& PBKV-LAE  & $146.67$ & $11.83$ & $44.91$ & $119.13$ & $9.57$  & $41.18$ \\
& PBKV-HE   & $108.86$ & $8.95$  & $66.01$ & $80.62$  & $6.53$  & $64.15$ \\

& \cellcolor{green!10}Full PBKV & \cellcolor{green!10}$102.60$ & \cellcolor{green!10}$8.22$ & \cellcolor{green!10}$69.10$ & \cellcolor{green!10}$76.15$ & \cellcolor{green!10}$5.50$ & \cellcolor{green!10}$68.29$ \\
\midrule
\multirow{4}{*}{\shortstack[l]{\textbf{SWE-bench}\\\textbf{+ AutoGen}\\(with retry\\ \textit{loops})}}
& LRU       & $271.90$ & $5.04$ & $46.34$ & $214.74$ & $4.18$ & $48.53$ \\
& PBKV-LAE  & $233.86$ & $3.73$ & $59.51$ & $177.50$ & $3.26$ & $60.41$ \\
& PBKV-HE   & $177.22$ & $2.46$ & $75.31$ & $141.53$ & $2.19$ & $73.36$ \\
& \cellcolor{green!10}Full PBKV & \cellcolor{green!10}$160.18$ & \cellcolor{green!10}$2.27$ & \cellcolor{green!10}$79.94$ & \cellcolor{green!10}$118.61$ & \cellcolor{green!10}$2.05$ & \cellcolor{green!10}$77.07$ \\
\midrule
\multirow{3}{*}{\shortstack[l]{\textbf{FinanceBench}\\\textbf{+ CrewAI}\\(\textit{static})}}
& LRU       & $130.34$ & $12.93$ & $27.95$ & $96.11$ & $9.02$ & $26.58$ \\
& \cellcolor{orange!15}KVFlow    & \cellcolor{orange!15}$101.57$ & \cellcolor{orange!15}$9.91$ & \cellcolor{orange!15}$39.87$ & \cellcolor{orange!15}$80.13$ & \cellcolor{orange!15}$7.42$ & \cellcolor{orange!15}$39.65$ \\
& \cellcolor{green!10}Full PBKV & \cellcolor{green!10}$80.53$ & \cellcolor{green!10}$8.00$ & \cellcolor{green!10}$53.44$ & \cellcolor{green!10}$65.01$ & \cellcolor{green!10}$6.06$ & \cellcolor{green!10}$55.07$ \\
\bottomrule
\vspace{-25pt}
\end{tabular}
}
\end{table}

\vspace{-2pt}
\stitle{Observations.}
We take HoVer + LangChain workload as a representative case.
Under $K$=3 and $\gamma$=0.7, Table~\ref{tab:main} shows that PBKV reduces end-to-end latency and per-agent TTFT by up to $1.85\times$ and $2.03\times$ over LRU, respectively. This directly confirms that PBKV delivers substantial gains in both system efficiency and quality of service (i.e., QoS).
The KV-Cache hit rate exposes the source of these gains:
\textbf{(i) LRU} achieves a hit rate below 30\% on both models,
confirming that it fails to capture the workflow-level reuse.
\textbf{(ii) PBKV-LAE}, which merely identifies and evicts \textit{retired cache} first, already lifts the hit rate by up to $1.66\times$ over LRU, showing that even simple \textit{lifecycle-awareness} is highly effective.
\textbf{(iii) PBKV-HE} additionally incorporates the workflow predictor and reuse scoring, pushing the hit rate to 66.01\%, because it assesses the reuse value of active cache through \textit{cross-workflow aggregation} and \textit{multi-step lookahead}. As a result, eviction decisions remain informed even after the retired cache is exhausted, and the policy never degenerates back to LRU.
\textbf{(iv) Full PBKV} introduces \textit{speculative prefetching} on top of PBKV-HE and attains a hit rate of 69.10\% ($2.55\times$ over LRU).

\vspace{-2pt}
We further note that the hit rates are higher near the start and the end of each run, when the pressure is low.
Restricted to the steady-state interior, LRU's hit rate drops to $\sim$10\% while full PBKV's sustains $\sim$50\%.
Their behavior is detailed in Section~\ref{subsec:mechanism}.
Results of other tests exhibit similar trends.

\vspace{-2pt}
\stitle{Ablation Study.} Comparing the three PBKV variants, the contributions of Lifecycle-Aware Eviction, Hierarchical Eviction, and speculative prefetching are \textit{incremental}. The step from PBKV-LAE to PBKV-HE is pronounced, whereas the step from PBKV-HE to full PBKV is \textit{modest}.
This modest gap stems from two factors: (i) the prediction errors of the multi-step predictor on realistic dynamic workflows, and (ii) our deliberately \textit{conservative} prefetching design (Section~\ref{sec:prefetch}), which refuses to gamble known-valuable active cache for speculative gains.
Appendix~\ref{app:aggressive-prefetch} evaluates an aggressive alternative that evicts active cache to free prefetch space, confirming that our conservative variant delivers stable benefits while avoiding the backfire of aggressive prefetching under poor predictions.

\vspace{-2pt}
\stitle{Performance on Static Workflow.} Table~\ref{tab:main} also reports the results under a static workflow (FinanceBench + CrewAI). KVFlow infers the steps-to-execution from the predefined call graph and prioritizes the cache accordingly, yielding a moderate improvement over LRU. PBKV achieves a \textit{further} gain, which we attribute to two design choices. \textit{(i) Lifecycle-aware eviction.} Like LRU, KVFlow can also mistakenly evict active cache that will still be reused, whereas PBKV's lifecycle-aware eviction prevents this misjudgment. \textit{(ii) Cross-workflow reuse aggregation.} KVFlow prioritizes each cache node by the \textit{minimum} steps-to-execution across all related workflows, which captures \textit{when} it will be reused but not how \textit{popular} it is. PBKV instead \textit{aggregates} the predicted reuse contribution from every active workflow accessing the node, so the score of popular cache grows naturally.

\vspace{-3pt}
\stitle{Remark on Static Workflows.} Given the deterministic agent invocation order in \textit{static} workflows, PBKV could be further tuned by safely relaxing several mechanisms (e.g., adopting more aggressive prefetching and disabling the multi-step confidence decay $\gamma$). However, we retain these mechanisms in our evaluation for compatibility and consistency across both dynamic and static workloads.

\begin{figure}[t]
  \begin{subfigure}[b]{0.34\textwidth}
    \centering
    \includegraphics[width=\linewidth]{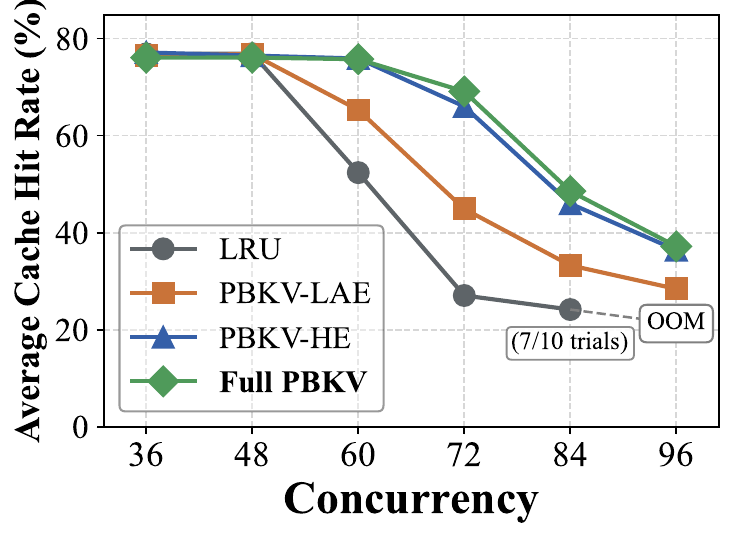}
    \vspace{-15pt}
    \caption{Average cache hit rate vs.\ concurrency. Excessive concurrency causes LRU to run out of memory (OOM).}
    \label{fig:micro-concurrency}
  \end{subfigure}
  \hfill
  \begin{subfigure}[b]{0.37\textwidth}
    \centering
    \includegraphics[width=\linewidth]{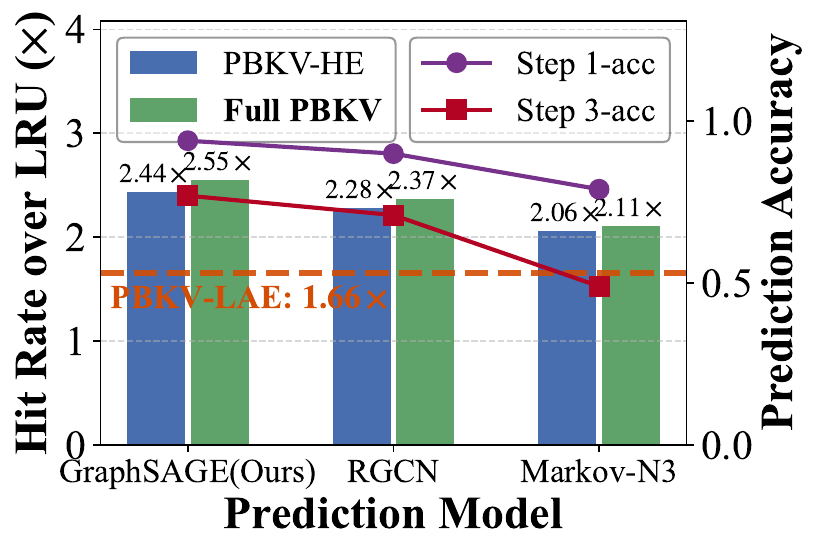}
    \vspace{-15pt}
    \caption{Prediction accuracy of different backbones (lines) and the corresponding hit-rate ratio of PBKV over LRU (bars).}
    \label{fig:micro-predictor-quality}
  \end{subfigure}
  \hfill
  \begin{subfigure}[b]{0.25\textwidth}
    \centering
    \includegraphics[width=\linewidth]{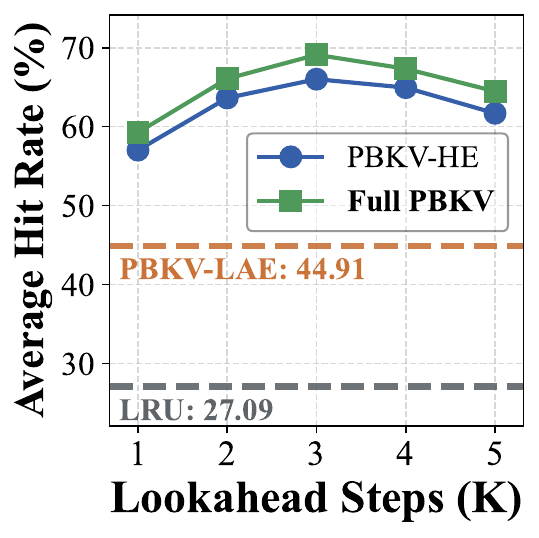}
    \vspace{-15pt}
    \caption{Average cache hit rate vs.\ lookahead horizon $K$ (i.e., prediction step).}
    \label{fig:micro-lookahead}
  \end{subfigure}

    \vspace{-3pt}
  \caption{Sensitivity analysis of PBKV and its variants on HoVer + LangChain with Qwen3-32B.}
  \label{fig:micro}
\vspace{-15pt}
\end{figure}

\vspace{-5pt}
\subsection{Sensitivity Analysis}
\vspace{-5pt}
\label{sec:sensitivity}

In this section, we vary several configurations to characterize their impact on PBKV and its variants. All experiments below are conducted on the HoVer + LangChain workload with Qwen3-32B.

\vspace{-3pt}
\stitle{Concurrency.}
It directly governs GPU memory pressure. As shown in Figure~\ref{fig:micro-concurrency}, all hit rates degrade as concurrency rises, yet PBKV consistently and greatly outperforms LRU.
Below 60 workflows, evictions are too rare to differentiate any policy.
At the other extreme, LRU exhausts memory and crashes in 3/10 trials at concurrency 84 and 10/10 at 96.
We therefore adopt a concurrency of 72 for HoVer + LangChain in the main experiments to evaluate all policies under reasonable pressure.

\vspace{-3pt}
\stitle{Predictor Backbone.} PBKV treats the predictor as a pluggable component. We replace our GraphSAGE backbone with R-GCN and a third-order Markov model (Markov-N3), each trained on the same 1$K$ traces. As shown in Figure~\ref{fig:micro-predictor-quality}, the variants exhibit lower prediction accuracy, which translates monotonically into lower cache hit rates.
Even so, the weakest backbone still strictly outperforms the prediction-free PBKV-LAE, confirming that PBKV is robust to predictor quality and that a reasonable predictor typically delivers additional benefits on top of lifecycle awareness.

\vspace{-3pt}
\stitle{Lookahead Horizon.}
As shown in Figure~\ref{fig:micro-lookahead}, the best hit rate occurs at $K$=3: a smaller horizon induces myopic eviction, whereas a larger one dilutes the score with inaccurate long-range predictions. The gap between full PBKV and PBKV-HE stays nearly constant across $K$ by design, since prefetching ranks candidates by the one-step value alone (Section~\ref{sec:prefetch}) and is decoupled from $K$.

\vspace{-3pt}
\stitle{Additional Studies.} Due to space constraints, we defer the following studies to the appendix:
(i) the accuracy of more predictor variants and how training scale affects it (Appendix~\ref{app:predictor});
(ii) sensitivity to the confidence decay coefficient $\gamma$ (Appendix~\ref{app:gamma-sensitivity});
(iii) aggressive vs.\ conservative prefetching across varying prediction accuracies (Appendix~\ref{app:aggressive-prefetch});
(iv) scheduling overhead (Appendix~\ref{app:scheduler-overhead});
and (v) a theoretical analysis of PBKV's graceful degradation under poor prediction (Appendix~\ref{app:smoothness}).

\vspace{-5pt}
\subsection{Why PBKV Works: Analysis of Cache Behavior.}
\vspace{-3pt}
\label{subsec:mechanism}

To understand the sources of PBKV's gains, we trace the KV-Cache hit rate of each policy throughout a full run in Figure~\ref{fig:hit-rate-curve}. Key events (marked by dashed vertical lines) divide the run into several phases:
\substitle{(i) Warm-up (0-30\,s).} With ample GPU memory, all policies rapidly exceed 80\% without eviction, \textit{confirming our motivation} in Section~\ref{sec:Intro} that multi-agent workflows expose abundant reuse potential.
\substitle{(ii) Onset of memory pressure (30-55\,s).} Around 30\,s, GPU memory saturates and forces evictions, causing LRU's hit rate (red) to drop sharply.
PBKV-LAE (orange) follows the same trajectory, as no workflow has terminated to release retired cache.
In contrast, PBKV-HE (blue) and full PBKV (green) decline only \textit{gradually}, since their scoring identifies low-value \textit{active} cache for eviction.
\substitle{(iii) Retired-cache window (55-80\,s).} From 55\,s onward, the first-completed workflows release retired cache.
PBKV-LAE immediately capitalizes on this signal and recovers visibly.
PBKV-HE and PBKV enjoy the same benefit through their hierarchical eviction.
Thereafter, LRU stays around a 10\% hit rate and only recovers near the end when the draining request pool relieves pressure.
\substitle{(iv) Retired-cache drained (80-115\,s).} Once the retired cache is drained, PBKV-LAE reverts to LRU.
Thereafter, its hit rate recovers only \textit{briefly} when batches of workflows release retired cache concurrently.
PBKV-HE and PBKV also decline but remain well above LRU, thanks to their scoring.
\substitle{(v) Steady-state serving (115-245\,s).} PBKV and PBKV-HE greatly outperform others, showing the lookahead score remains effective under prolonged pressure.
PBKV further improves PBKV-HE by a \textit{modest but consistent} margin, reflecting the bounded but stable benefit of \textit{conservative} prefetching.
\substitle{(vi) Tail phase (after 245\,s).} As the request pool drains, eviction pressure subsides and the hit rate climbs back to roughly 90\% by completion (around 285\,s).
Full PBKV enters and exits this phase \textit{first},
achieving the lowest average job completion time and the highest throughput among all policies.

\begin{figure}[t]
  \centering
    \centering
    \includegraphics[width=\textwidth]{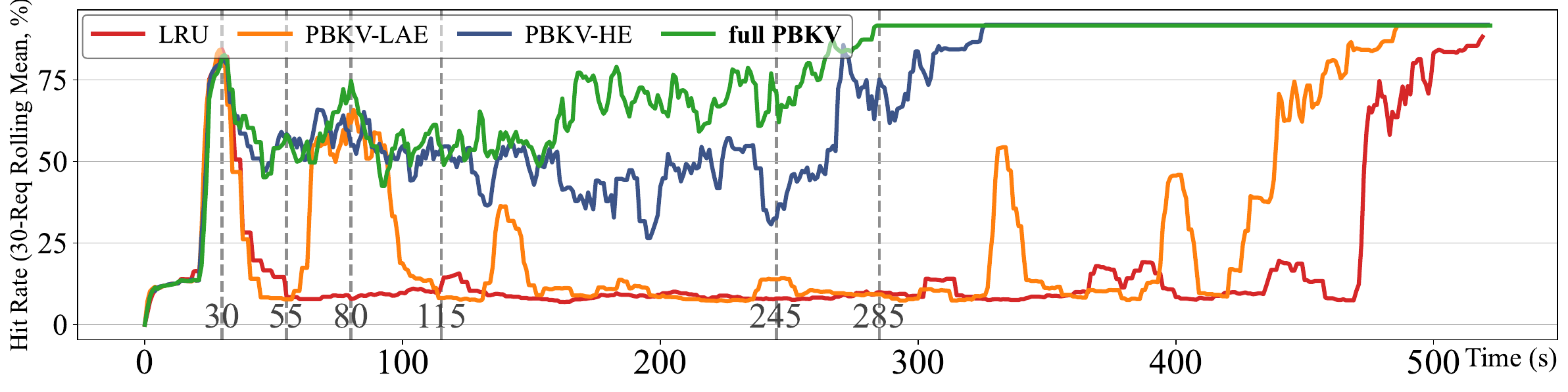}
    \vspace{-18pt}
    \caption{KV-Cache hit rate of each policy over time on the HoVer + LangChain workload, served by Qwen3-14B under a concurrency limit of 72. Dashed vertical lines mark key events.}
    \label{fig:hit-rate-curve}
\vspace{-15pt}
\end{figure}

\vspace{-3pt}
\stitle{Summary.} This phase-by-phase analysis reveals the \textit{complementary} roles of PBKV's three mechanisms. Lifecycle-aware eviction provides a deterministic improvement whenever retired cache is available; hierarchical eviction sustains improvements under prolonged pressure via multi-step lookahead and cross-workflow aggregation; and speculative prefetching contributes a bounded but stable improvement on top. Together, they enable PBKV to maintain a substantial and sustained lead over LRU, directly translating into lower latency (as shown in Table~\ref{tab:main}) and higher throughput.

\vspace{-10pt}
\section{Related Work}
\vspace{-5pt}
\label{sec:related_work}
\stitle{KV-Cache Management for Multi-Agent Serving.}
Modern inference engines (e.g., vLLM~\citep{vllm-kwon2023efficient} and SGLang~\citep{sglang_NEURIPS2024_724be447}) manage KV-Cache via paged
attention and radix-tree prefix caching under a default LRU policy. Orthogonal efforts such as HiCache~\citep{sglang_hicache_2025} and LMCache~\citep{liu2025lmcache} extend this with hierarchical GPU-CPU-disk offloading, while KVCOMM~\citep{ye2026kvcomm} and CacheBlend~\citep{yao2025cacheblend} target non-prefix reuse. All of them are oblivious to workflows.
Some recent systems try to close this gap, Continuum~\citep{li2026continuum} targets ReAct-style single-agent tool calling, pinning KV entries with a dynamic TTL learned from per-tool duration statistics, whereas KVFlow~\citep{pan2025kvflow} assumes a predefined Agent Step Graph and uses `steps-to-execution' scores for eviction and prefetching.
They are not well-suited to realistic multi-agent workflows, which are typically dynamic and involve runtime-dependent loops. PBKV fills this gap by driving \emph{workflow-level} KV-Cache management with multi-step \textit{prediction} for each workflow.

\vspace{-3pt}
\stitle{Agentic Workflow Prediction.}
A growing body of work predicts the structure of agentic workflows.
Parrot~\citep{lin2024parrot} introduces semantic variables to recover inter-request dependencies,
Ayo~\citep{tan2025ayo} compiles LLM applications into primitive-level dataflow graphs for end-to-end optimization,
and Autellix~\citep{luo2025autellix} elevates programs to first-class scheduling units.
Such abstractions assume a \textit{(near-)static} DAG or dataflow graph and struggle to faithfully express the conditional loops common in practical multi-agent workflows.
Another thread of work instead reasons about workflow structure online through speculative execution; e.g.,
PASTE~\citep{sui2026act_paste} exploits recurring tool-call patterns and predictable data dependencies to speculatively invoke tools,
and Speculative Actions~\citep{ye2026speculative} generalizes speculation to agent actions.
These methods accommodate dynamic workflows but predict only the next step, which can lead to \textit{myopic eviction} when driving cache eviction; and autoregressive rollout also accumulates errors. Therefore, we develop a multi-step predictor customized for KV-Cache management.

\vspace{-10pt}
\section{Conclusion}
\vspace{-5pt}

We present PBKV for realistic dynamic multi-agent workflows.
First, PBKV proposes guidelines for a dedicated \textit{predictor}, i.e., (i) fusing complementary signals from the global call graph and per-request prefill embeddings, and (ii) forecasting over a \textit{multi-step} horizon to avoid myopic decisions.
Second, PBKV adopts a \textit{conservative} policy, i.e., (i) hierarchical eviction that combines lifecycle-aware reclamation with lookahead, cross-workflow-aggregated scoring, and (ii) conservative prefetching that consumes only otherwise-idle GPU space and PCIe bandwidth.
\textbf{Theoretically}, we prove that PBKV degrades gracefully under prediction error. \textbf{Empirically}, PBKV significantly outperforms LRU across models and workloads, and also outperforms the SOTA method on static workflows.

\bibliographystyle{unsrtnat}
\bibliography{references}

\appendix
\section{Main Results with Standard Deviations.}
\label{app:full-results}

\begin{table}[t]
\centering
\caption{Full results with standard deviations across multi-agent workloads. Each cell reports mean $\pm$ standard deviation over multiple runs. The \textit{static} workload (FinanceBench+CrewAI) is included for comparison with \smash{\colorbox{orange!15}{KVFlow}}, which assumes a predefined static call graph. Concurrency limits are 72 for HoVer+LangChain, 24 for SWE-bench+AutoGen, and 48 for FinanceBench+CrewAI.}
\label{apptab:full-results-std}
\resizebox{\linewidth}{!}{
\setlength{\tabcolsep}{4pt}
\begin{tabular}{lllccc}
\toprule
Model & Workload & Policy & Latency (s) $\downarrow$ & Per-Agent TTFT (s) $\downarrow$ & Cache Hit Rate (\%) $\uparrow$ \\
\midrule
\multirow{11}{*}{Qwen3-32B}
& \multirow{4}{*}{\shortstack[l]{HoVer\\+ LangChain}}
  & LRU                                  & $189.66 \pm 4.70$  & $16.65 \pm 0.78$ & $27.09 \pm 1.16$ \\
& & PBKV-LAE                             & $146.67 \pm 6.73$  & $11.83 \pm 0.74$ & $44.91 \pm 2.20$ \\
& & PBKV-HE                              & $108.86 \pm 6.93$  & $8.95 \pm 0.74$  & $66.01 \pm 2.79$ \\
& & \cellcolor{green!10}Full PBKV        & \cellcolor{green!10}$102.60 \pm 7.28$ & \cellcolor{green!10}$8.22 \pm 0.71$ & \cellcolor{green!10}$69.10 \pm 2.23$ \\
\cmidrule(l){2-6}
& \multirow{4}{*}{\shortstack[l]{SWE-bench\\+ AutoGen}}
  & LRU                                  & $271.90 \pm 9.14$  & $5.04 \pm 0.26$ & $46.34 \pm 2.01$ \\
& & PBKV-LAE                             & $233.86 \pm 12.70$ & $3.73 \pm 0.44$ & $59.51 \pm 2.92$ \\
& & PBKV-HE                              & $177.22 \pm 10.36$ & $2.46 \pm 0.37$ & $75.31 \pm 2.57$ \\
& & \cellcolor{green!10}Full PBKV        & \cellcolor{green!10}$160.18 \pm 10.89$ & \cellcolor{green!10}$2.27 \pm 0.35$ & \cellcolor{green!10}$79.94 \pm 3.26$ \\
\cmidrule(l){2-6}
& \multirow{3}{*}{\shortstack[l]{FinanceBench\\+ CrewAI\\(\textit{static})}}
  & LRU                                  & $130.34 \pm 5.37$ & $12.93 \pm 0.51$ & $27.95 \pm 0.82$ \\
& & \cellcolor{orange!15}KVFlow          & \cellcolor{orange!15}$101.57 \pm 5.27$ & \cellcolor{orange!15}$9.91 \pm 0.57$ & \cellcolor{orange!15}$39.87 \pm 1.43$ \\
& & \cellcolor{green!10}Full PBKV        & \cellcolor{green!10}$80.53 \pm 5.06$ & \cellcolor{green!10}$8.00 \pm 0.48$ & \cellcolor{green!10}$53.44 \pm 2.07$ \\
\midrule
\multirow{11}{*}{Qwen3-14B}
& \multirow{4}{*}{\shortstack[l]{HoVer\\+ LangChain}}
  & LRU                                  & $139.59 \pm 3.13$ & $10.90 \pm 0.35$ & $28.79 \pm 1.04$ \\
& & PBKV-LAE                             & $119.13 \pm 6.04$ & $9.57 \pm 0.58$  & $41.18 \pm 3.17$ \\
& & PBKV-HE                              & $80.62 \pm 5.13$  & $6.53 \pm 0.34$  & $64.15 \pm 1.99$ \\
& & \cellcolor{green!10}Full PBKV        & \cellcolor{green!10}$76.15 \pm 5.95$ & \cellcolor{green!10}$5.50 \pm 0.42$ & \cellcolor{green!10}$68.29 \pm 2.32$ \\
\cmidrule(l){2-6}
& \multirow{4}{*}{\shortstack[l]{SWE-bench\\+ AutoGen}}
  & LRU                                  & $214.74 \pm 6.98$ & $4.18 \pm 0.28$ & $48.53 \pm 3.02$ \\
& & PBKV-LAE                             & $177.50 \pm 9.05$ & $3.26 \pm 0.51$ & $60.41 \pm 4.62$ \\
& & PBKV-HE                              & $141.53 \pm 7.57$ & $2.19 \pm 0.43$ & $73.36 \pm 3.95$ \\
& & \cellcolor{green!10}Full PBKV        & \cellcolor{green!10}$118.61 \pm 8.06$ & \cellcolor{green!10}$2.05 \pm 0.47$ & \cellcolor{green!10}$77.07 \pm 3.71$ \\
\cmidrule(l){2-6}
& \multirow{3}{*}{\shortstack[l]{FinanceBench\\+ CrewAI\\(\textit{static})}}
  & LRU                                  & $96.11 \pm 4.88$ & $9.02 \pm 0.44$ & $26.58 \pm 0.86$ \\
& & \cellcolor{orange!15}KVFlow          & \cellcolor{orange!15}$80.13 \pm 4.16$ & \cellcolor{orange!15}$7.42 \pm 0.41$ & \cellcolor{orange!15}$39.65 \pm 1.84$ \\
& & \cellcolor{green!10}Full PBKV        & \cellcolor{green!10}$65.01 \pm 4.12$ & \cellcolor{green!10}$6.06 \pm 0.40$ & \cellcolor{green!10}$55.07 \pm 2.39$ \\
\bottomrule
\end{tabular}
}
\end{table}

In our main experiments, we evaluate LRU, KVFlow (only on the static workload due to its inherent limitation), and our PBKV along with its variants on two LLMs and three workloads. Each setting is repeated for 10 runs, with mean and standard deviation reported in Table~\ref{apptab:full-results-std}.
\section{Predictor Details: Architecture and Training}
\label{app:predictor-details}
This appendix complements Section~\ref{sec:prediction} with the full design of the predictor, including the rationale behind the GraphSAGE backbone, the detailed architecture, and the training procedure.

\begin{figure}[t]
  \centering
    \centering
    \includegraphics[width=0.8\textwidth]{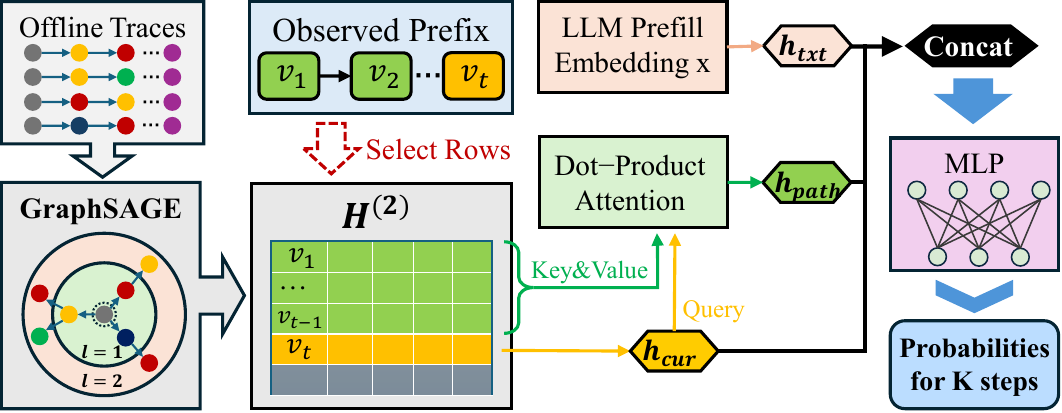}
    \caption{\textbf{Architecture of the predictor.} It fuses a topology-aware agent embedding from GraphSAGE ($h_{cur}$), an attention-based workflow prefix summary ($h_{path}$), and a semantic signal reused from prefill ($h_{txt}$), then jointly predicts the next $K$ agent probability distributions via an MLP head.}
    \label{appfig:predictor}
\end{figure}

\stitle{Overview.}
The predictor adopts \textit{GraphSAGE}~\cite{hamilton2017inductive_Graphsage} as its backbone, fuses \textit{multiple complementary} signals, and jointly forecasts the next \textit{few} invoked agents (i.e., Multi-Step Prediction).

\stitle{Why GraphSAGE?}
Our predictor design follows two desiderata. \textit{First}, the predictor must serve \textit{all} possible workflows across every admissible agent transition, so it is desirable to leverage the structural priors encoded in the global call graph $G$. \textit{Second}, the dynamic nature of agentic workflows requires the predictor to generalize to unseen subgraphs encountered at runtime.

Guided by these desiderata, we adopt \textbf{GraphSAGE} as the backbone of predictor. \textit{First}, GraphSAGE learns node representations directly on the graph through neighborhood sampling and aggregation, thereby preserving the structural priors encoded in agent transitions, which an MLP would otherwise discard. \textit{Second}, GraphSAGE is an \emph{inductive} framework, which performs forward inference on unseen workflow prefix without retraining on the entire graph. This property aligns naturally with the dynamic nature of agentic workflows, and is preferable to transductive alternatives (e.g., GCN~\cite{kipf2017semisupervised_GCN}) that rely on a fixed training graph.
For completeness, we also implement and evaluate several baseline predictors (including MLPs, GCNs, and Markov models), detailed results are in Appendix~\ref{app:predictor}.

\stitle{Architecture.}
Figure~\ref{appfig:predictor} illustrates the architecture of predictor.
To forecast the future invocations of an active workflow, the predictor fuses three complementary streams as described below:

\emph{(i) Topology-aware agent representation.}
We assign each agent $v \in V$ a learnable embedding $\mathbf{e}_v \in \mathbb{R}^{d}$, stacked into $\mathbf{H}^{(0)} \in \mathbb{R}^{|V|\times d}$. $A \in \mathbb{R}^{|V|\times|V|}$ denotes the row-normalized forward \textit{transition matrix} estimated from the offline traces. We then apply two GraphSAGE propagation layers:

\begin{equation}
    \mathbf{H}^{(\ell)} = \mathrm{ReLU}\!\left( \bigl[\, \mathbf{H}^{(\ell-1)} \,\|\, A\mathbf{H}^{(\ell-1)} \,\bigr](\mathbf{W}^{(\ell)})^\top\right), \quad \mathbf{W}^{(\ell)} \in \mathbb{R}^{d \times 2d}, \quad \ell = 1, 2
\end{equation}

where $\|$ denotes feature-wise concatenation. Thus, for each agent, the resulting $\mathbf{H}^{(2)}$ encodes both its own identity (through $\mathbf{H}$) and the structural priors induced by its neighborhood in $G$ (through $A\mathbf{H}$).

\emph{(ii) Attention-based history aggregation.}
While $\mathbf{H}^{(2)}$ equips each agent with a topology-aware representation, it does not yet capture the state of a \emph{workflow}. In fact, two workflows residing at the same agent $v_t$ may have arrived through very different prefixes. To leverage such prefix information, we summarize the prefix via scaled \textit{dot-product attention}, using the current agent representation $\mathbf{h}_{\mathrm{cur}} = \mathbf{H}^{(2)}_{v_t}$ as the query and the prefix representations as both keys and values:
\begin{equation}
    \alpha_i \;=\; \frac{\exp\!\bigl( (\mathbf{W}_q \mathbf{h}_{\mathrm{cur}})^{\top} \mathbf{H}^{(2)}_{v_i} / \sqrt{d} \bigr)}{\sum_{j<t} \exp\!\bigl( (\mathbf{W}_q \mathbf{h}_{\mathrm{cur}})^{\top} \mathbf{H}^{(2)}_{v_j} / \sqrt{d} \bigr)}, \qquad \mathbf{h}_{\mathrm{path}} \;=\; \sum_{i<t} \alpha_i\, \mathbf{H}^{(2)}_{v_i}
\end{equation}
Compared with naive mean pooling, attention allows the predictor to upweight those prefix agents that are most informative of the current state.

\emph{(iii) Semantic signal from prefill.}
Graph topology and workflow history alone do not distinguish requests that share the same prefix but diverge in intent. To capture such \textit{request-specific} semantics, we reuse the post-norm hidden state $\mathbf{x}$ of the last prefill token and project it through a single linear layer, i.e., $\mathbf{h}_{\mathrm{txt}} = \mathrm{ReLU}(\mathbf{W}_t \mathbf{x})$. As $\mathbf{x}$ is a by-product of prefill, this signal is obtained essentially for free.

\stitle{Multi-step prediction head.}
The above streams are concatenated and fed into a two-layer MLP with dropout, which jointly emits logits over agents for each of the next $K$ steps.
Emitting all $K$ logits from a shared backbone in one forward pass further avoids autoregressive error accumulation and bounds per-invocation latency to that of a single inference.

\stitle{Training Strategy.}
We train the predictor on offline invocation traces.
For each position $t$ within a workflow, we pair the prefix $(v_1,\ldots,v_{t})$ and its prefill embedding with the future agents $(v_{t+1},\ldots,v_{t+K})$ as labels.
The token $\langle\mathrm{END}\rangle$ is treated as a regular target so that the termination probability $p_{w, \langle\mathrm{END}\rangle}$ is properly learned, and only the padding positions after $\langle\mathrm{END}\rangle$ in a target window are masked from the cross-entropy loss. The resulting predictor has roughly 350K parameters (orders of magnitude smaller than the served LLM) and processes a batch of 1,024 requests in 1.56 ms, rendering its runtime overhead negligible.
Detailed hyperparameter settings and training configurations are provided in the supplementary code.

\section{Workflow Prediction}
\label{app:predictor}

This appendix supplements Section~\ref{sec:prediction} with three additional
results: (i) the scaling behavior of our predictor across training-set
sizes (Section~\ref{app:scaling}), (ii) a comparison against alternative
predictor families (Section~\ref{app:full}), and (iii) the sensitivity
of prediction accuracy to which transformer layer the prefill semantic
signal is extracted from (Section~\ref{app:layer}).

\stitle{Setup.}
All experiments use the HoVer~\citep{jiang2020hover} dataset with the
LangChain~\citep{langchain2022} agent framework. We report top-$1$
accuracy at horizons $k{=}1, 2, 3$ (denoted $s_1, s_2, s_3$), averaged
over $10$ random seeds.

\subsection{Scaling Behavior}
\label{app:scaling}

\begin{figure}[t]
  \centering
  \includegraphics[width=0.9\textwidth]{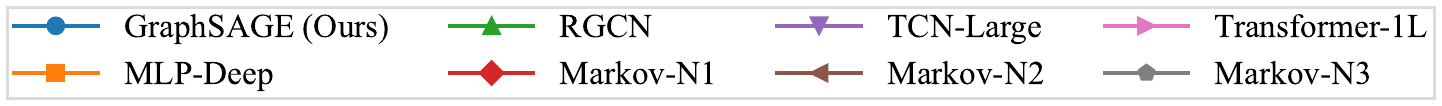}\\
  \begin{subfigure}[b]{0.32\textwidth}
    \centering
    \includegraphics[width=\textwidth]{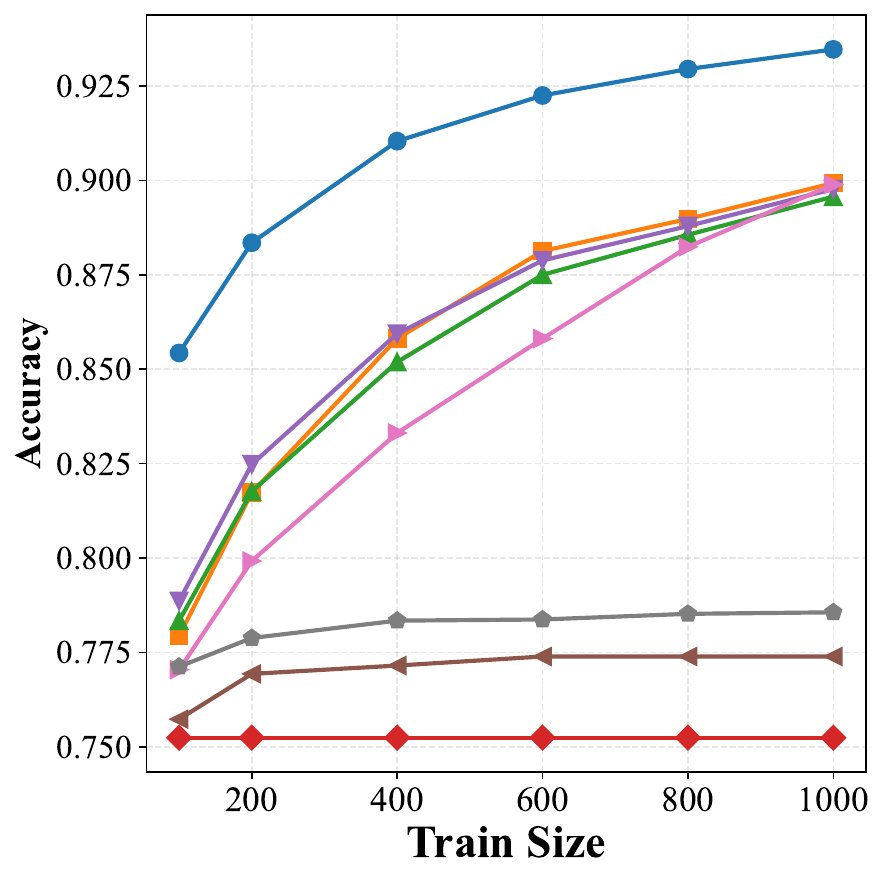}
    \caption{Prediction Accuracy of Step~1.}
  \end{subfigure}
  \begin{subfigure}[b]{0.32\textwidth}
    \centering
    \includegraphics[width=\textwidth]{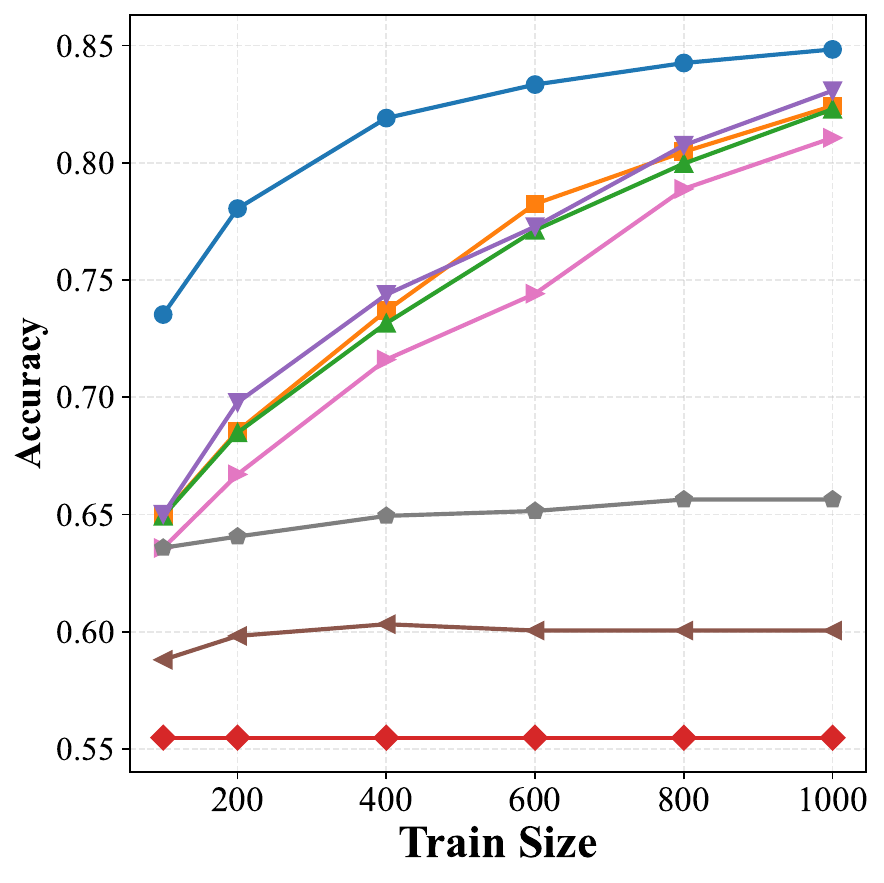}
    \caption{Prediction Accuracy of Step~2.}
  \end{subfigure}
  \begin{subfigure}[b]{0.32\textwidth}
    \centering
    \includegraphics[width=\textwidth]{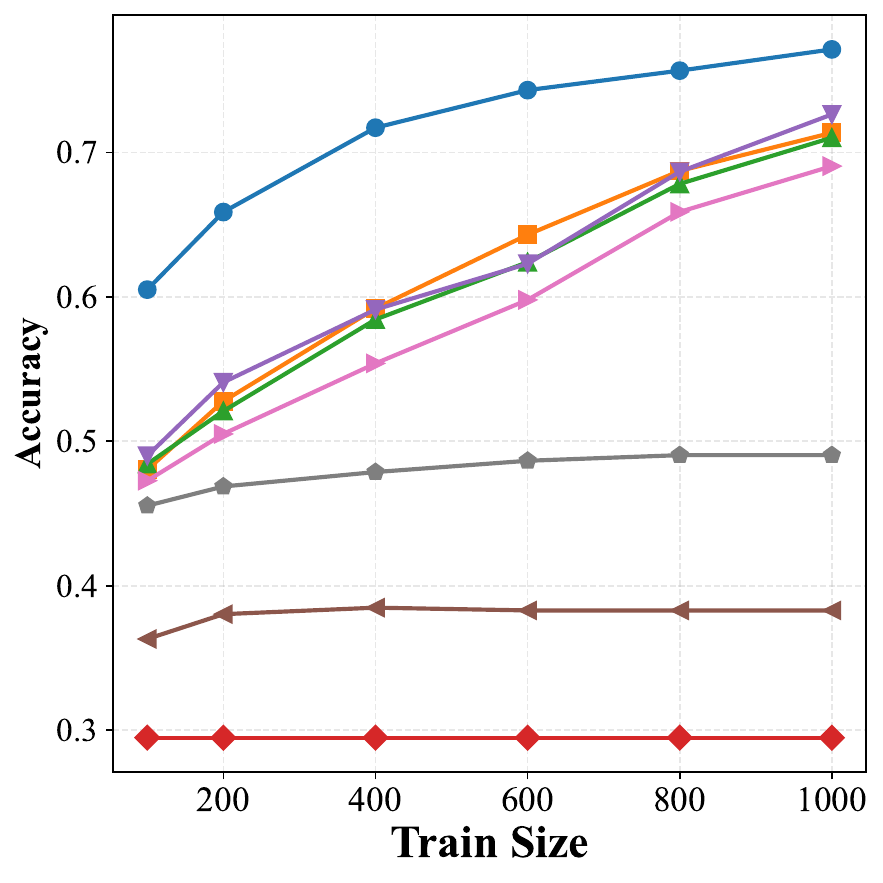}
    \caption{Prediction Accuracy of Step~3.}
  \end{subfigure}
  \caption{Top-$1$ prediction accuracy as a function of training-set size
  at horizons $k{=}1, 2, 3$. Our predictor leads every baseline at every
  train size shown.}
  \label{fig:scaling}
\end{figure}

Figure~\ref{fig:scaling} plots top-$1$ accuracy against the number of
training traces. Two observations emerge. First, learning-based
predictors scale with data while the Markov baseline saturates within
$100$ traces, justifying predicting from a learned representation rather
than from tabulated transition probabilities. Second, our predictor
leads every learning-based baseline at every horizon and the lead
widens as the horizon grows.

\subsection{Comparison Across Predictor Families}
\label{app:full}

\begin{table}[t]
  \centering
  \small
  \caption{Predictor comparison at $1000$ training traces, grouped by
  architectural family (mean~$\pm$~std over $10$ seeds).}
  \label{tab:full}
  \begin{tabular}{llccc}
    \toprule
    Family & Predictor & $s_1$ & $s_2$ & $s_3$ \\
    \midrule
    \textbf{GraphSAGE (Ours)}         &  & $\mathbf{0.935}\pm 0.007$ & $\mathbf{0.848}\pm 0.012$ & $\mathbf{0.771}\pm 0.019$ \\
    \midrule
    R-GCN        & \textsc{R-GCN}           & $0.896\pm 0.008$          & $0.823\pm 0.011$          & $0.710\pm 0.018$          \\
    \midrule
    \multirow{2}{*}{Transformer}
                 & \textsc{1L}              & $0.899\pm 0.008$          & $0.811\pm 0.023$          & $0.690\pm 0.034$          \\
                 & \textsc{2L}              & $0.896\pm 0.011$          & $0.810\pm 0.022$          & $0.688\pm 0.040$          \\
    \midrule
    \multirow{2}{*}{TCN}
                 & \textsc{small}           & $0.901\pm 0.006$          & $0.824\pm 0.015$          & $0.717\pm 0.019$          \\
                 & \textsc{large}           & $0.898\pm 0.006$          & $0.831\pm 0.007$          & $0.726\pm 0.019$          \\
    \midrule
    \multirow{3}{*}{MLP}
                 & \textsc{tiny}            & $0.786\pm 0.003$          & $0.656\pm 0.004$          & $0.494\pm 0.003$          \\
                 & \textsc{small}           & $0.889\pm 0.004$          & $0.801\pm 0.012$          & $0.684\pm 0.021$          \\
                 & \textsc{deep}            & $0.899\pm 0.006$          & $0.824\pm 0.016$          & $0.714\pm 0.024$          \\
    \midrule
    Linear Probe & \textsc{post-norm}            & $0.840\pm 0.026$          & $0.732\pm 0.026$          & $0.576\pm 0.036$          \\
    \midrule
    \multirow{2}{*}{kNN}
                 & \textsc{$k{=}5$}         & $0.839\pm 0.008$          & $0.699\pm 0.007$          & $0.564\pm 0.008$          \\
                 & \textsc{$k{=}20$}        & $0.817\pm 0.005$          & $0.707\pm 0.007$          & $0.576\pm 0.006$          \\
    \midrule
    \multirow{3}{*}{Markov}
                 & \textsc{$n{=}1$}         & $0.752\pm 0.000$          & $0.555\pm 0.000$          & $0.295\pm 0.000$          \\
                 & \textsc{$n{=}2$}         & $0.774\pm 0.000$          & $0.601\pm 0.000$          & $0.383\pm 0.000$          \\
                 & \textsc{$n{=}3$}         & $0.786\pm 0.000$          & $0.656\pm 0.000$          & $0.490\pm 0.000$          \\
    \bottomrule
  \end{tabular}
\end{table}

Table~\ref{tab:full} compares our predictor against every baseline
family we implemented: parameter-free baselines (Markov, kNN), a linear
probe on prefill states, MLPs, sequence models (Transformer~\citep{vaswani2017attention}, TCN~\citep{bai2018empirical_TCN}), and
a relational graph encoder (R-GCN~\citep{schlichtkrull2018modeling_RGCN}). Our predictor outperforms the
strongest variant of every other family on all three horizons. The lead widens further at longer
horizons, consistent with the multi-step argument that long-horizon
prediction benefits most from a rich summary of the workflow prefix.

\subsection{Choice of the Layer for Semantic Extraction}
\label{app:layer}

\begin{figure}[t]
  \centering
  \begin{subfigure}[b]{\textwidth}
    \centering
    \includegraphics[width=\textwidth]{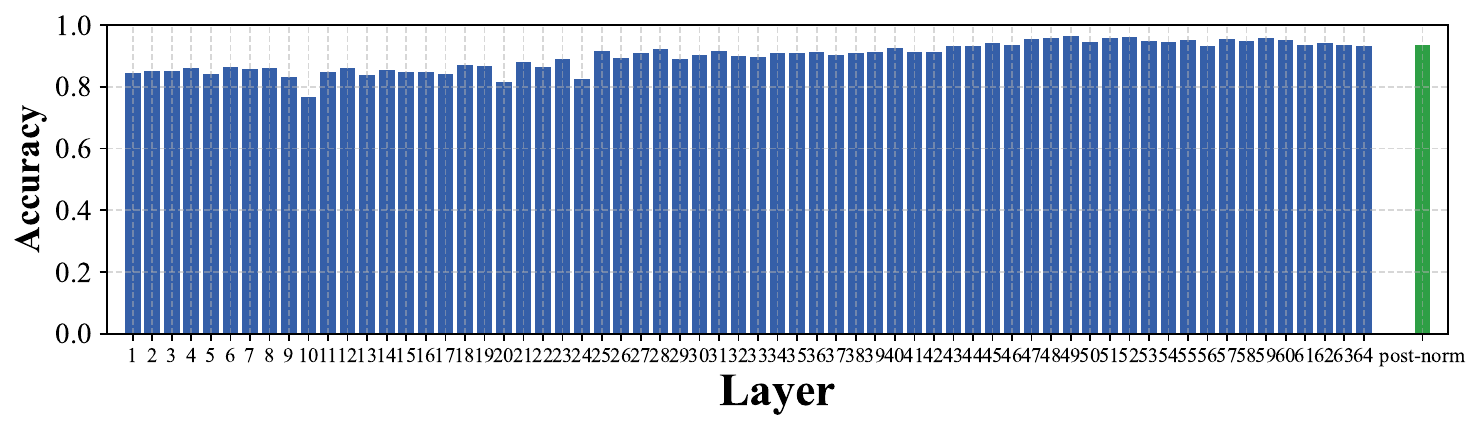}
    \caption{Prediction Accuracy of Step~1.}
  \end{subfigure}
  \begin{subfigure}[b]{\textwidth}
    \centering
    \includegraphics[width=\textwidth]{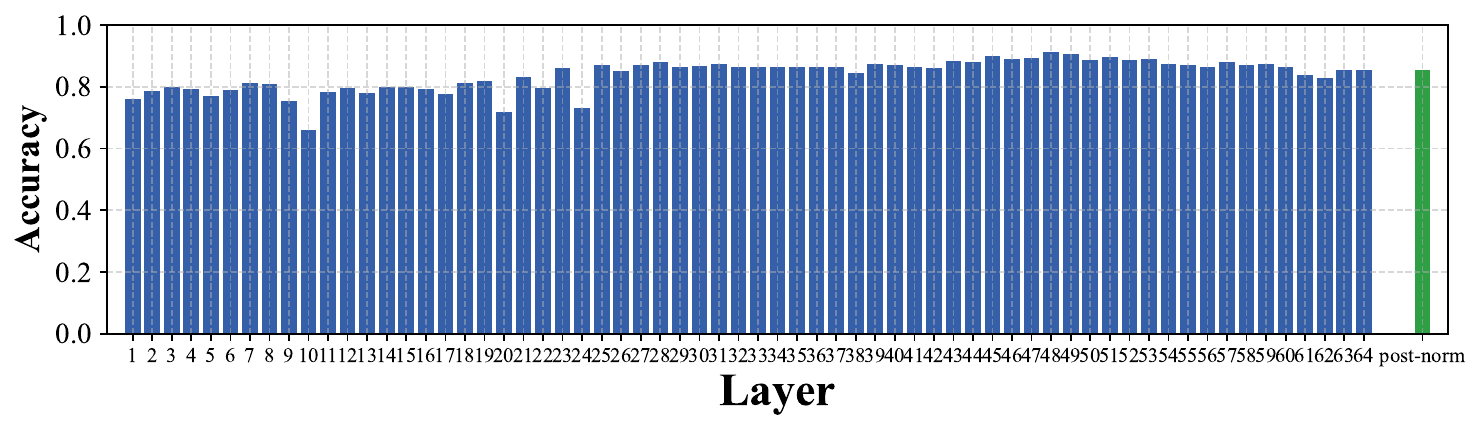}
    \caption{Prediction Accuracy of Step~2.}
  \end{subfigure}
  \begin{subfigure}[b]{\textwidth}
    \centering
    \includegraphics[width=\textwidth]{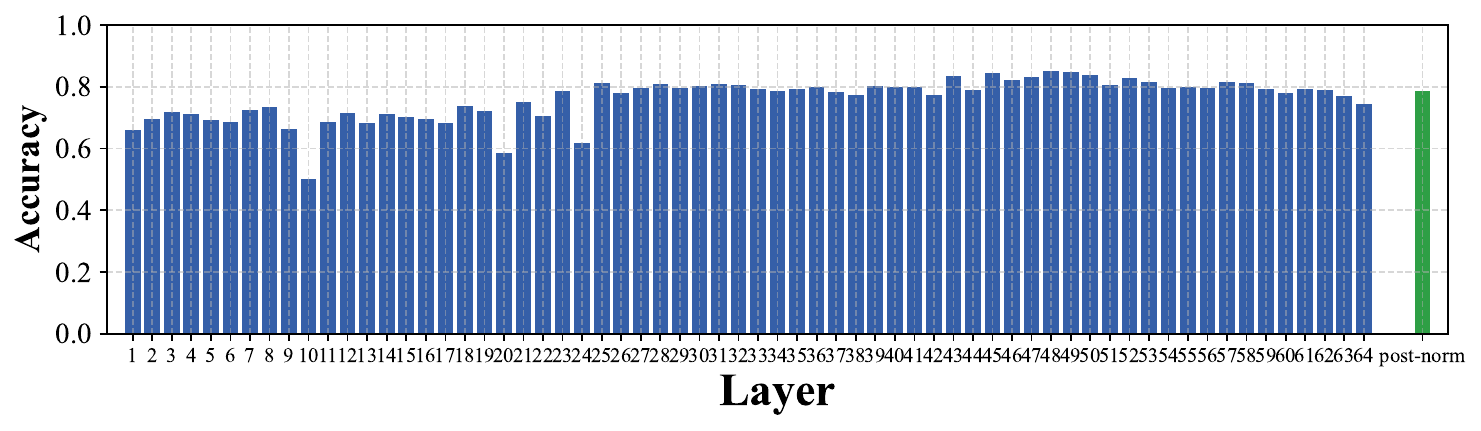}
    \caption{Prediction Accuracy of Step~3.}
  \end{subfigure}

  \caption{Top-$1$ prediction accuracy as a function of the layer from
  which the prefill semantic signal ${h}_{\text{txt}}$ is extracted.
  The served LLM is Qwen3-32B, which exposes $64$ transformer block
  outputs ($\ell{=}1, \ldots, 64$) followed by the post-norm hidden
  state \textit{post-norm} that is fed into the output head.}
  \label{fig:layer}
\end{figure}

The prefill semantic signal ${h}_{\text{txt}}$ described in
Section~\ref{sec:prediction} is extracted from a single layer of the served
LLM, and a natural question is which layer to choose. Prior
studies~\citep{shahout2025dont,skean2025layer,ulanovski2026improving} report that the optimal layer for hidden
state extraction is uncertain and depends jointly on the served model
and the downstream task. We therefore study this choice empirically on
Qwen3-32B by training the predictor with ${h}_{\text{txt}}$ taken
from each transformer block output ($\ell{=}1, \ldots, 64$) and from
the post-norm hidden state \textit{post-norm}, holding the rest of the
architecture and the training protocol fixed.

Figure~\ref{fig:layer} reports the resulting $s_1, s_2, s_3$ accuracies.
Consistent with prior observations, the optimal layer varies across
horizons and is highly specific to the served LLM and the workload at
hand: deploying our system on a different LLM or workload would require
re-running the layer scan to identify a new optimum. To preserve
compatibility with arbitrary served LLMs and workloads, and to keep the
predictor self-contained, we extract ${h}_{\text{txt}}$ from
\textit{post-norm} in all other experiments. The post-norm hidden state
is the standard input to the LLM's output head, and is therefore
universally accessible across LLMs without extra instrumentation,
regardless of how many transformer blocks the model contains.

\section{Sensitivity to the Confidence Decay Coefficient}
\label{app:gamma-sensitivity}
\begin{figure}[ht]
  \centering
    \centering
    \includegraphics[width=0.3\textwidth]{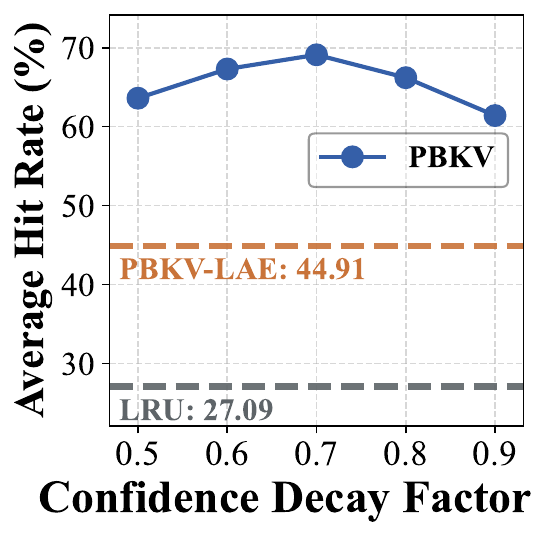}
    \caption{Average cache hit rate vs.\ confidence decay coefficient $\gamma$.}
    \label{appfig:gamma-sensitivity}
\end{figure}

The score computation uses a coefficient $\gamma$ to discount the contribution of farther-step predictions, reflecting their decreasing reliability. We study its sensitivity on Qwen3-32B under the HoVer+LangChain workload. As shown in Figure~\ref{appfig:gamma-sensitivity}, $\gamma=0.7$ achieves the best performance. A smaller $\gamma$ can cause the score to decay too rapidly, suppressing lookahead and degrading toward myopic behavior; while a larger $\gamma$ will over-weight unreliable predictions at distant steps, leading to misguided eviction decisions.
\section{Aggressive vs.\ Conservative Prefetching under Varying Prediction Accuracy}
\label{app:aggressive-prefetch}

\subsection{Motivation and Setup}

Section~\ref{sec:prefetch} argues that PBKV adopts a conservative prefetching principle: under dynamic workflows, the uncertainty of prediction makes ``displacing known-valuable active cache in exchange for speculatively valuable cache'' an asymmetric trade-off between deterministic cost and probabilistic benefit. This appendix empirically validates this design choice.

\stitle{Aggressive variant.} We construct an aggressive prefetching variant that relaxes only the space constraint of Section~\ref{sec:prefetch}, while keeping the rest of the design unchanged (i.e., the predictor, the scoring function, the activation only on pure-decode batches, and the PCIe bandwidth budget $S_{bw}$). Specifically, beyond the conservative budget $S_a$, the aggressive variant is additionally permitted to displace at most $\rho \cdot S_{\text{total}}$ of active cache to make room for prefetching, where $S_{\text{total}}$ denotes the total GPU cache capacity. Active nodes are evicted in ascending order of their lookahead score, ensuring that the lowest-scored nodes are displaced first.

\stitle{Prediction noise injection.} To isolate the effect of prediction accuracy, we apply a controlled perturbation to the predictor's output. Let $P_w^{(k)}$ denote the GraphSAGE next-step distribution for workflow $w$ at horizon $k$, and let $U$ denote the uniform distribution over agents. The perturbed distribution is:
\begin{equation}
\tilde{P}_w^{(k)} = (1-\lambda) \, P_w^{(k)} + \lambda \, U,
\end{equation}
where $\lambda \in [0, 1]$ denotes the noise level. $\lambda=0$ corresponds to the original predictor, while $\lambda=1$ degenerates to a uniform (i.e., random) prediction. The perturbation is applied to the predictor outputs consumed by both eviction and prefetching, simulating a system-wide degradation of the prediction signal.

\stitle{Workload and metrics.} We follow the main experimental setup, i.e., the HoVer + LangChain workload with Qwen3-32B at a concurrency of 72. Each configuration is repeated 10 times, and we report the mean and standard deviation of the average cache hit rate.

\subsection{Experiment 1: Sensitivity to the Prefetching Space Budget}

We first sweep $\rho$ under the noise-free setting $\lambda=0$ to identify the optimal aggressiveness of the variant.

\begin{figure}[ht]
\centering
\includegraphics[width=0.4\linewidth]{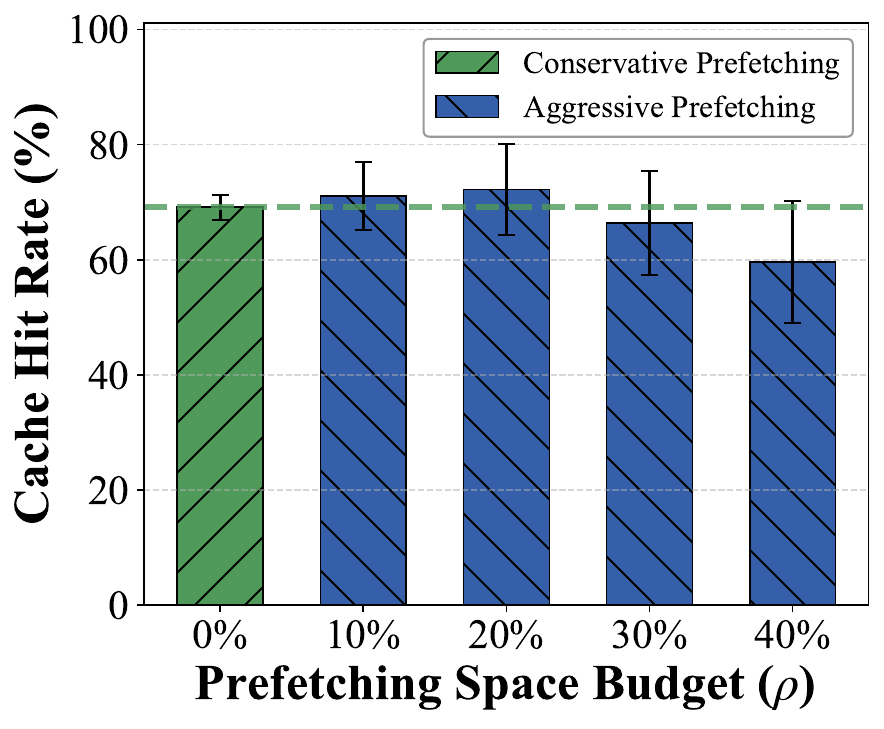}
\caption{Cache hit rate of conservative and aggressive prefetching across different prefetching space budgets $\rho$, under the noise-free setting ($\lambda=0$). Error bars denote standard deviation over 10 trials.}
\label{fig:prefetching-rho}
\end{figure}

As shown in Figure~\ref{fig:prefetching-rho}, the hit rate of aggressive prefetching first increases and then decreases with $\rho$, peaking at $\rho=20\%$ with roughly a 3-percentage-point margin over conservative prefetching. As $\rho$ further grows, the hit rate falls below the conservative baseline, indicating that the cumulative cost of mistakenly displacing active cache outweighs the gain from prefetching.

More notably, the standard deviation grows monotonically with $\rho$ even under the noise-free setting. At the optimal point $\rho^*=20\%$, the standard deviation is already $\sim 3.5\times$ that of conservative prefetching. This indicates that, although aggressive prefetching can deliver expected gains, it harms the stability of the system.

In the following analysis, we use $\rho^*=20\%$, at which aggressive prefetching achieves its best performance.

\subsection{Experiment 2: Sweeping the Prediction Noise}

Under $\rho = \rho^*$, we sweep the noise level and compare PBKV-HE (no-prefetching reference), conservative prefetching (the default in PBKV), and aggressive prefetching.

\begin{figure}[ht]
\centering
\includegraphics[width=0.5\linewidth]{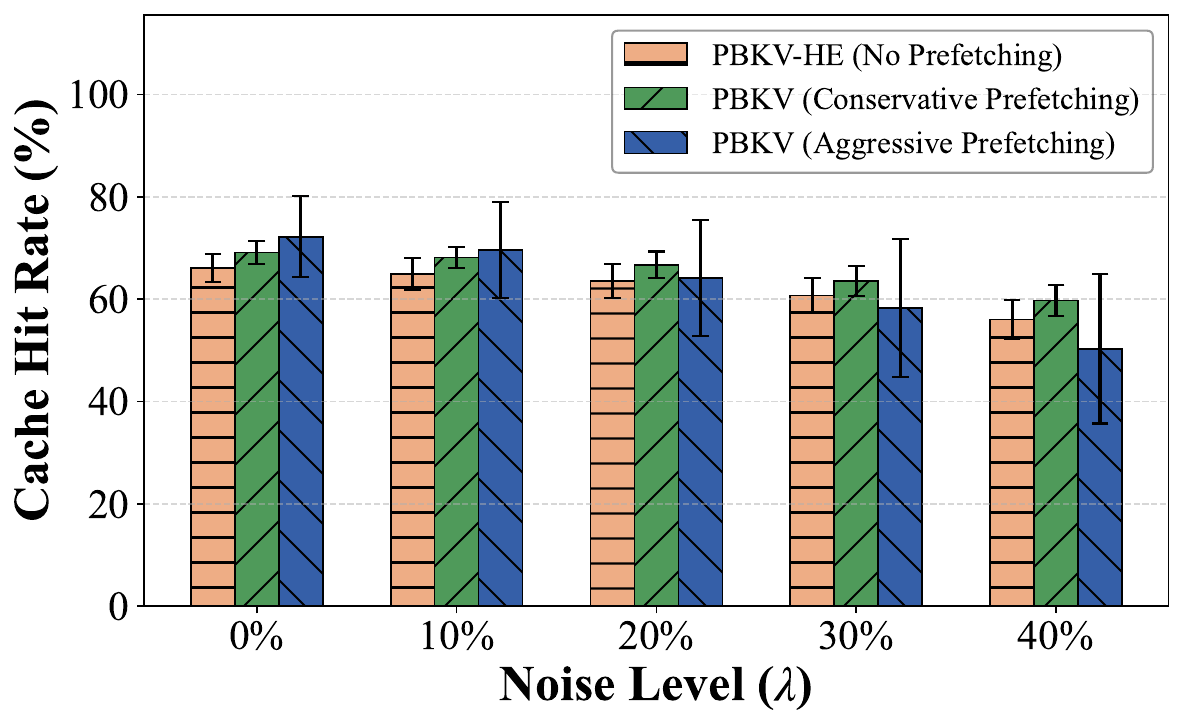}
\caption{Cache hit rate of PBKV-HE, conservative prefetching, and aggressive prefetching across different noise levels $\lambda$. Error bars denote standard deviation over 10 trials.}
\label{fig:prefetching-noise}
\end{figure}

As shown in Figure~\ref{fig:prefetching-noise}, two key observations emerge:

\noindent\textbf{(i) The advantageous regime of aggressive prefetching is narrow.} Only at $\lambda=0$ and $\lambda=10\%$ does aggressive prefetching outperform conservative prefetching, and the margin shrinks rapidly as the noise increases. At $\lambda=20\%$, aggressive prefetching is already overtaken by conservative prefetching. In practice, the realistic operating regime of aggressive prefetching is even narrower.

\noindent\textbf{(ii) Aggressive prefetching falls below PBKV-HE under moderate-to-high noise.} At $\lambda=30\%$, aggressive prefetching already underperforms PBKV-HE, which performs no prefetching at all; at $\lambda=40\%$ the gap further widens. This shows that aggressive prefetching becomes counterproductive under inaccurate predictions, with its cost outweighing its benefit. In contrast, conservative prefetching consistently outperforms PBKV-HE across all noise levels, reflecting its lower-bound guarantee.

\subsection{Justification of the Design Choice}

In summary, aggressive prefetching only achieves a marginal and unstable advantage over conservative prefetching when the prediction accuracy is sufficiently high (i.e., the noise level is below 20\%). However, dynamic multi-agent workflows are inherently hard to predict, especially under complex agent frameworks. Therefore, PBKV adopts conservative prefetching to obtain stable gains.
\section{Scheduler Overhead Analysis}
\label{app:scheduler-overhead}

A natural concern with prediction-driven cache management is whether the additional scheduler logic erodes the cache-hit-rate gains. We instrument every PBKV scheduler hook in our SGLang implementation and measure per-call wall-clock latency over a full run of the HoVer+LangChain workload on Qwen3-32B at a concurrency of 72, identical to the operating point used throughout Section~\ref{subsec:main-results}.
\begin{table}[ht]
\centering
\caption{Per-call latency of PBKV's scheduler components on Qwen3-32B with the HoVer+LangChain workload at concurrency 72. EvictDecision selects victim nodes on memory pressure; PredictInfer runs one forward pass of the GraphSAGE predictor; ScoreUpdate refreshes a cache node's heap entry after a new prediction; PrefetchDecision ranks host-resident candidates under the $(S_a, S_{bw})$ budget; DecodeStep is one batched decode iteration, included as a reference. \textbf{Note the unit difference}: ScoreUpdate is in microseconds (\textmu s); all other components in milliseconds (ms).}
\label{tab:overhead}
\begin{tabular}{lrr}
\toprule
\textbf{Component} & \textbf{Mean} & \textbf{P99} \\
\midrule
EvictDecision    & 0.44~ms        & 0.81~ms \\
PredictInfer     & 1.18~ms        & 1.68~ms \\
ScoreUpdate      & 1.53~\textbf{\textmu s} & 2.69~\textbf{\textmu s} \\
PrefetchDecision & 0.64~ms        & 1.60~ms \\
\midrule
DecodeStep       & 12.34~ms       & 43.14~ms \\
\bottomrule
\end{tabular}
\end{table}

Table~\ref{tab:overhead} reports the mean and P99 latency of EvictDecision, PredictInfer, ScoreUpdate, and PrefetchDecision, with DecodeStep included as a reference for relative cost.

Among them, EvictDecision, ScoreUpdate, and PrefetchDecision run on the CPU scheduler thread and operate purely on PBKV's heap-based bookkeeping structures, concurrent with the GPU decode kernel. PredictInfer executes on the GPU but is issued on a separate CUDA stream that overlaps with the ongoing decode kernel. None of the four therefore sit on the GPU critical path. ScoreUpdate in particular is essentially free at $1.53~\mu$s on average, precisely what allows per-node scores to be refreshed at high frequency without serializing the scheduler.
Eviction and prefetch decisions both average well under $1$~ms ($0.44$~ms and $0.64$~ms), and even predictor inference -- the dominant component -- averages just $1.18$~ms, less than $10\%$ of a mean decode step ($12.34$~ms), and the other components fall well below this. These overheads do not show up in end-to-end performance: PBKV achieves a $1.85\times$ speedup over LRU on this workload (Section~\ref{subsec:main-results}).
\section{Smoothness Analysis}
\label{app:smoothness}

\stitle{Overview.}
This appendix establishes the smoothness pillar of PBKV's algorithms-with-predictions guarantee. We organize the analysis into four layers. Lemma~\ref{lem:emc} identifies the per-node score $\mathrm{SCORE}(c)$ as the expected discounted miss count attributable to evicting $c$, thereby grounding the proxy cost in a system-relevant quantity. Lemma~\ref{lem:lipschitz} shows that this score is Lipschitz-continuous in the predicted distributions. Corollary~\ref{cor:rank} translates this score-level continuity into pairwise ranking stability. Theorem~\ref{thm:regret} further lifts it to a cost-level regret bound. All Lipschitz-type results share a single multiplier with a $K$-independent upper bound $1/\bigl(2(1-\gamma)\bigr)$ that is also independent of $|\mathcal{W}_{\mathrm{act}}|$ and the cache size, which justifies treating the predictor as a pluggable module: any future improvement in predictor accuracy directly tightens all bounds.

\subsection{Setup}
\label{app:smoothness:setup}

\stitle{Probabilistic model.}
We posit a ground-truth conditional distribution under which each active workflow $w$ produces, at every step $k$, a true distribution $P_w^{(k)} \in \Delta^{|V|}$ over $V \cup \{\langle\mathrm{END}\rangle\}$, conditioned on the workflow having not terminated before step $k$. Here $\Delta^{|V|}$ denotes the $|V|$-dimensional probability simplex over the $|V|+1$ outcomes in $V \cup \{\langle\mathrm{END}\rangle\}$, following the standard topological convention. The predictor outputs an estimate $\widehat{P}_w^{(k)}$ on the same simplex. Consistent with the sequential agent execution adopted by mainstream multi-agent frameworks (e.g., LangChain~\cite{langchain2022}, Microsoft AutoGen~\cite{wu2024autogen}, MetaGPT~\cite{hong2024metagpt}), we model each active workflow as invoking exactly one agent per step (or terminating), so that $P_w^{(k)}$ is the marginal distribution of the step-$k$ invocation conditioned on survival to step $k$. Let $p_{w,\langle\mathrm{END}\rangle}^{(j)}$ denote the conditional termination probability at step $j$. The cumulative survival factor follows the standard hazard-product decomposition $s_w^{(k)} = \prod_{j=1}^{k-1}\!\bigl(1 - p_{w,\langle\mathrm{END}\rangle}^{(j)}\bigr)$. Throughout, $\mathcal{W}_{\mathrm{act}}(c)$ denotes the set of active workflows associated with node $c$ at the moment the eviction decision is made.

\stitle{Eviction-only abstraction.}
Following the classical formulation of caching analysis pioneered by Belady~\cite{belady1966study} and adopted by recent algorithms-with-predictions work~\cite{lykouris2021competitive,rohatgi2020near}, we evaluate eviction policies under the \emph{eviction-only abstraction}, i.e., a node, once evicted, remains out of the cache for the entire $K$-step horizon over which the cost is measured. This abstraction isolates the quality of the eviction decision from orthogonal storage-hierarchy mechanisms (e.g., HiCache re-admission on miss), which are properties of the underlying system rather than of the eviction policy itself. PBKV's eviction policy operates on the same abstraction and is agnostic to whether a second-tier storage is present. The composition of our analysis with HiCache is discussed at the end of Section~\ref{app:smoothness:regret}.

\stitle{Access-indicator semantics.}
For each active workflow $w \in \mathcal{W}_{\mathrm{act}}(c)$, the access indicator $A_w(c) \in \{0,1\}^{|V|}$ marks the set of agents within $w$ that access $c$, denoted $\mathcal{O}_w(c) := \{a \in V : A_w(c)_a = 1\}$. Two cases arise. (i) For private cache, $|\mathcal{W}_{\mathrm{act}}(c)|$ is small and $A_w(c)$ is typically one-hot, since the owning agent of $c$ is determined by the upstream agent that produced it within $w$. (ii) For global cache (e.g., system prompts or tool/agent descriptions), $|\mathcal{W}_{\mathrm{act}}(c)|$ may be large and each $A_w(c)$ may have multiple ones, reflecting that distinct agents within the same workflow all access the shared prefix. The score formula in Eq.~\eqref{eq:multistep-score} treats both cases under a single mathematical form, with the cross-workflow outer sum protecting popular-prefix nodes by design (Section~\ref{sec:eviction}). For notational convenience, we adopt the convention that $A_w(c)$ is implicitly zero-padded on the $\langle\mathrm{END}\rangle$ coordinate, so that the inner product
\begin{equation}
\label{eq:zero-pad}
A_w(c)^{\top} P_w^{(k)} \;=\; \sum_{a \in V} A_w(c)_a\, P_w^{(k)}(a) \;=\; P_w^{(k)}\!\bigl(\mathcal{O}_w(c)\bigr)
\end{equation}
is dimensionally consistent and equals the probability that workflow $w$ invokes some agent in $\mathcal{O}_w(c)$ at step $k$.

\stitle{Past access equals future hit.}
The radix-tree prefix structure adopted by SGLang~\cite{sglang_NEURIPS2024_724be447} ensures that $A_w(c)$, defined from the observed access history, also characterizes future hits within the same workflow. Specifically, a cache node $c$ corresponds to a fixed token prefix; if agent $a$ has invoked the radix path through $c$ in $w$, then any subsequent invocation of $a$ within $w$ (e.g., under a retry loop) prepends the same system prompt and upstream context, and therefore traverses the same prefix and hits $c$ provided $c$ remains cached. Conversely, an agent that has never accessed $c$ does not have $c$ on its execution path and will not hit $c$ in future invocations within $w$. Hence, throughout the analysis, ``$w$ would hit $c$ at step $k$ if $c$ remained cached'' is equivalent to ``$w$ invokes some $a \in \mathcal{O}_w(c)$ at step $k$''.

\stitle{Unit-size abstraction and the role of node size.}
A natural question is why $\mathrm{SCORE}(c)$ in Eq.~\eqref{eq:multistep-score} does not contain an explicit $|c|$ factor. This is a deliberate design choice: for a cache node $c$ of size $|c|$, both the space cost of retaining it and the wall-clock cost of a miss on it scale linearly with $|c|$, since a miss triggers the re-prefill (or PCIe transfer under HiCache) of $|c|$ tokens. The two factors of $|c|$ therefore cancel out when the quantity of interest is value per unit space. By construction, $\mathrm{SCORE}(c)$ measures exactly this per-unit-space value, which is also the canonical ranking criterion of fractional-knapsack-style cache replacement: ranking by per-unit-space value, rather than by absolute value, is what minimizes total miss cost under a space budget. Multiplying $\mathrm{SCORE}(c)$ by $|c|$ would conflate ``valuable'' with ``large'' and bias eviction against small-but-hot nodes, which is the opposite of the intended behavior.

For the regret analysis, we adopt the standard unit-size abstraction~\cite{belady1966study,lykouris2021competitive,rohatgi2020near}, under which each node contributes an equal volume to the eviction budget and the budget-$B$ eviction reduces to selecting $B$ lowest-scoring nodes. This abstraction is consistent with the size-heterogeneous setting via a standard reduction in which each node of size $|c|$ is treated as $|c|$ unit-size virtual nodes sharing the same per-unit-space score $\mathrm{SCORE}(c)$, whereby cardinality-constrained selection over virtual nodes recovers size-constrained selection over original nodes. Since radix-tree nodes are page-sized in practice and $|c| \ll B$, the boundary effect of integer-budget rounding is at most one node and does not affect the asymptotic behavior of the bound. The Lipschitz and regret bounds below therefore transfer to the size-heterogeneous setting without modification of the multiplier $1/\bigl(2(1-\gamma)\bigr)$.

\stitle{Perturbation model.}
Following standard practice in algorithms-with-predictions analyses for caching~\cite{lykouris2021competitive,rohatgi2020near}, we measure prediction error via the per-step $\ell_1$ deviation and its node-local discount-weighted aggregate:
\begin{equation}
\label{eq:perturb}
\epsilon_w^{(k)} := \bigl\lVert P_w^{(k)} - \widehat{P}_w^{(k)} \bigr\rVert_1, \qquad
\epsilon_c^{\gamma} := \sum_{w\,\in\,\mathcal{W}_{\mathrm{act}}(c)} \sum_{k=1}^{K} \gamma^{k-1}\, \epsilon_w^{(k)}.
\end{equation}
Throughout, we let $\mathrm{SCORE}(c)$ and $\widehat{\mathrm{SCORE}}(c)$ denote the per-node scores computed from $\{P_w^{(k)}\}$ and $\{\widehat{P}_w^{(k)}\}$ respectively, both via Eq.~\eqref{eq:multistep-score}.

\subsection{From Score to Expected Miss Count}
\label{app:smoothness:emc}

We first establish the system-level meaning of $\mathrm{SCORE}(c)$, which grounds the subsequent regret analysis in a quantity of direct interest to caching theory.

\begin{lemma}[Score equals expected discounted miss count attributable to eviction]
\label{lem:emc}
Under the eviction-only abstraction (Section~\ref{app:smoothness:setup}), for any cache node $c$ that is evicted at the decision point and not re-admitted within the horizon,
\begin{equation}
\label{eq:emc}
\mathrm{EMC}(c) \;:=\; \sum_{k=1}^{K} \gamma^{k-1} \,\mathbb{E}\bigl[\text{misses on } c \text{ at step } k\bigr] \;=\; \mathrm{SCORE}(c),
\end{equation}
where the expectation is taken under the ground-truth distribution and the misses are aggregated across all $w \in \mathcal{W}_{\mathrm{act}}(c)$.
\end{lemma}

\begin{proof}
Fix a cache node $c$ and a workflow $w \in \mathcal{W}_{\mathrm{act}}(c)$. Let $\#\mathrm{miss}_w^{(k)}(c) \in \{0,1\}$ denote the indicator that $w$ incurs a miss on $c$ at step $k$. Since $w$ invokes exactly one agent per step (or terminates), and any invocation of an agent $a \in \mathcal{O}_w(c)$ would have traversed $c$ by the radix-tree prefix structure (Section~\ref{app:smoothness:setup}), we have
\begin{equation}
\#\mathrm{miss}_w^{(k)}(c) \;=\; \mathbf{1}\!\bigl[w \text{ survives to step } k \text{ and invokes some } a \in \mathcal{O}_w(c)\bigr],
\end{equation}
where the equality uses that, under the eviction-only abstraction, every such invocation is a miss. The events $\{\text{$w$ invokes $a$ at step $k$}\}_{a \in V}$ are mutually exclusive conditional on survival to step $k$, so
\begin{equation}
\Pr\!\bigl[w \text{ invokes some } a \in \mathcal{O}_w(c) \,\big|\, w \text{ survives to } k\bigr] \;=\; \sum_{a \in \mathcal{O}_w(c)} P_w^{(k)}(a) \;=\; A_w(c)^{\top} P_w^{(k)}.
\end{equation}
Multiplying by $\Pr[w \text{ survives to } k] = s_w^{(k)}$ gives $\mathbb{E}\!\bigl[\#\mathrm{miss}_w^{(k)}(c)\bigr] = s_w^{(k)} \cdot A_w(c)^{\top} P_w^{(k)}$. Summing across $w \in \mathcal{W}_{\mathrm{act}}(c)$ by linearity of expectation and discounting by $\gamma^{k-1}$ over $k = 1, \ldots, K$, comparing with Eq.~\eqref{eq:multistep-score} yields $\mathrm{EMC}(c) = \mathrm{SCORE}(c)$.
\end{proof}

\stitle{Implications.}
Lemma~\ref{lem:emc} provides three benefits. (i) It removes any circularity between the scoring rule and the cost it is evaluated against, since $\mathrm{SCORE}(c)$ is now identified with a quantity defined purely from the ground-truth distribution and the cache state. (ii) It places PBKV's analysis on the same footing as classical caching results~\cite{belady1966study}, where eviction policies are evaluated through their effect on miss counts, independent of any specific reload mechanism. (iii) It motivates the proxy cost adopted in Section~\ref{app:smoothness:regret} as a system-relevant rather than ad-hoc surrogate.

\subsection{Lipschitz Continuity of the Score}
\label{app:smoothness:lipschitz}

We next show that $\mathrm{SCORE}(c)$ depends continuously on the predicted distributions, with a Lipschitz constant whose dependence on the prediction horizon $K$ admits a uniform $K$-independent upper bound.

\begin{lemma}[Lipschitz continuity of the score]
\label{lem:lipschitz}
For any cache node $c$ and any pair of distribution sets $\{P_w^{(k)}\}$ and $\{\widehat{P}_w^{(k)}\}$ in $\Delta^{|V|}$,
\begin{equation}
\label{eq:lipschitz}
\Bigl| \mathrm{SCORE}(c) - \widehat{\mathrm{SCORE}}(c) \Bigr|
\;\leq\; \frac{1 - \gamma^{K}}{2(1 - \gamma)}\, \epsilon_c^{\gamma}
\;\leq\; \frac{\epsilon_c^{\gamma}}{2(1 - \gamma)}.
\end{equation}
The Lipschitz multiplier admits a $K$-independent upper bound $1/\bigl(2(1-\gamma)\bigr)$ that is also independent of $|\mathcal{W}_{\mathrm{act}}(c)|$ and the cache size; the dependence on these quantities is fully isolated in $\epsilon_c^{\gamma}$.
\end{lemma}

\begin{proof}
We bound the per-term deviation $\bigl| s_w^{(k)} A_w(c)^{\top} P_w^{(k)} - \widehat{s}_w^{(k)} A_w(c)^{\top} \widehat{P}_w^{(k)} \bigr|$ in two steps and then aggregate over $k$ and $w$.

\textit{Step 1 (per-term bound).}
The survival factor $s_w^{(k)}$ takes values in $[0,1]$ as a product of $[0,1]$ terms, and so does $\widehat{s}_w^{(k)}$. The standard telescoping identity for products of $[0,1]$-valued sequences gives
\begin{equation}
\label{eq:telescope}
\bigl| s_w^{(k)} - \widehat{s}_w^{(k)} \bigr| \;\leq\; \sum_{j=1}^{k-1} \bigl| p_{w,\langle\mathrm{END}\rangle}^{(j)} - \widehat{p}_{w,\langle\mathrm{END}\rangle}^{(j)} \bigr|.
\end{equation}
For any pair of probability distributions $P, Q$ on a finite outcome space and any subset $S$ of outcomes, the difference of subset-event probabilities is bounded by the total variation distance, $|P(S) - Q(S)| \leq \mathrm{TV}(P, Q) = \tfrac{1}{2}\|P - Q\|_1$. Applied to the singleton $S = \{\langle\mathrm{END}\rangle\}$, this yields $\bigl| p_{w,\langle\mathrm{END}\rangle}^{(j)} - \widehat{p}_{w,\langle\mathrm{END}\rangle}^{(j)} \bigr| \leq \tfrac{1}{2}\epsilon_w^{(j)}$, and substituting into~\eqref{eq:telescope} gives
\begin{equation}
\label{eq:survival-bound}
\bigl| s_w^{(k)} - \widehat{s}_w^{(k)} \bigr| \;\leq\; \frac{1}{2}\sum_{j=1}^{k-1} \epsilon_w^{(j)}.
\end{equation}
Similarly, by the zero-padding convention~\eqref{eq:zero-pad}, $A_w(c)^{\top} P_w^{(k)} = P_w^{(k)}\!\bigl(\mathcal{O}_w(c)\bigr) \in [0,1]$ is a subset-event probability under $P_w^{(k)}$, and the same holds for $A_w(c)^{\top} \widehat{P}_w^{(k)}$. Applying the same total-variation bound to the subset $S = \mathcal{O}_w(c)$,
\begin{equation}
\label{eq:access-bound}
\bigl| A_w(c)^{\top} P_w^{(k)} - A_w(c)^{\top} \widehat{P}_w^{(k)} \bigr| \;\leq\; \frac{1}{2}\,\epsilon_w^{(k)}.
\end{equation}
Combining \eqref{eq:survival-bound} and \eqref{eq:access-bound} via the elementary identity $|uv - \widehat{u}\widehat{v}| \leq |v - \widehat{v}| + |u - \widehat{u}|$ for $u,\widehat{u},v,\widehat{v} \in [0,1]$ yields
\begin{equation}
\label{eq:per-term}
\Bigl| s_w^{(k)} A_w(c)^{\top} P_w^{(k)} - \widehat{s}_w^{(k)} A_w(c)^{\top} \widehat{P}_w^{(k)} \Bigr| \;\leq\; \frac{1}{2}\sum_{j=1}^{k} \epsilon_w^{(j)}.
\end{equation}

\textit{Step 2 (aggregation).}
Multiplying \eqref{eq:per-term} by $\gamma^{k-1}$, summing over $k = 1,\dots,K$, and exchanging the order of summation between $k$ and $j$,
\begin{equation}
\sum_{k=1}^{K} \gamma^{k-1}\!\cdot\! \frac{1}{2}\sum_{j=1}^{k} \epsilon_w^{(j)}
\;=\; \frac{1}{2}\sum_{j=1}^{K} \gamma^{j-1} \epsilon_w^{(j)} \cdot \frac{1 - \gamma^{K-j+1}}{1-\gamma}
\;\leq\; \frac{1 - \gamma^{K}}{2(1-\gamma)} \sum_{j=1}^{K} \gamma^{j-1} \epsilon_w^{(j)},
\end{equation}
where the inner sum is a finite geometric series and the inequality uses $1 - \gamma^{K-j+1} \leq 1 - \gamma^{K}$ for $j \geq 1$ and $\gamma \in (0,1)$. Summing across $w \in \mathcal{W}_{\mathrm{act}}(c)$ yields the tighter form $(1-\gamma^{K})\epsilon_c^{\gamma}/\bigl(2(1-\gamma)\bigr)$ in Eq.~\eqref{eq:lipschitz}, and the looser corollary $\epsilon_c^{\gamma}/\bigl(2(1-\gamma)\bigr)$ follows from $1 - \gamma^{K} \leq 1$.
\end{proof}

\stitle{Locality remark.}
Only workflows in $\mathcal{W}_{\mathrm{act}}(c)$ contribute to $\epsilon_c^{\gamma}$, so prediction errors on workflows $w$ with $A_w(c) = 0$ leave $\mathrm{SCORE}(c)$ unchanged. This locality reflects the radix-tree prefix structure respected by the scoring rule and matches the design intent of Section~\ref{sec:eviction}.

\begin{corollary}[Ranking stability]
\label{cor:rank}
Let $c_1, c_2$ be two cache nodes with $\mathrm{SCORE}(c_1) > \mathrm{SCORE}(c_2)$, and let $\Delta := \mathrm{SCORE}(c_1) - \mathrm{SCORE}(c_2)$. If
\begin{equation}
\frac{1 - \gamma^{K}}{2(1 - \gamma)}\bigl(\epsilon_{c_1}^{\gamma} + \epsilon_{c_2}^{\gamma}\bigr) \;<\; \Delta,
\end{equation}
then $\widehat{\mathrm{SCORE}}(c_1) > \widehat{\mathrm{SCORE}}(c_2)$, and the eviction preference between $c_1$ and $c_2$ is preserved.
\end{corollary}

\begin{proof}
By Lemma~\ref{lem:lipschitz} applied to $c_1$ and $c_2$ separately, together with the triangle inequality.
\end{proof}

\subsection{Eviction Cost Regret}
\label{app:smoothness:regret}

We now lift the score-level continuity to a cost-level regret bound, which is the form prescribed by the smoothness desideratum of the algorithms-with-predictions framework~\cite{mitzenmacher2022algorithms}.

\stitle{Proxy cost.}
We adopt
\begin{equation}
\label{eq:proxy-cost}
\mathcal{L}(E) \;:=\; \sum_{c \in E} \mathrm{SCORE}(c),
\end{equation}
which, by Lemma~\ref{lem:emc}, equals the expected discounted miss count incurred by evicting the set $E$ under the eviction-only abstraction. This grounds the proxy cost in the canonical quantity studied since Belady~\cite{belady1966study} and isolates the eviction-policy quality from orthogonal storage-hierarchy concerns. The same surrogate has also been used in algorithms-with-predictions analyses for caching~\cite{lykouris2021competitive,rohatgi2020near}, where costs are likewise defined through predicted access patterns rather than realized latencies.

\stitle{Regret definition.}
For the regret analysis, we specialize $\{P_w^{(k)}\}$ to the ground-truth distributions and $\{\widehat{P}_w^{(k)}\}$ to the predictor outputs. Freeing a budget of $B$ nodes reduces to a cardinality-constrained selection (Section~\ref{app:smoothness:setup}). PBKV ranks candidate nodes by $\widehat{\mathrm{SCORE}}$ and evicts
\begin{equation}
\widehat{E}_B \;:=\; \arg\min_{|E| = B} \sum_{c \in E} \widehat{\mathrm{SCORE}}(c),
\end{equation}
whereas the cost-minimizing eviction set under the ground truth is $E_B^{\star} := \arg\min_{|E| = B} \mathcal{L}(E)$. The eviction cost regret of PBKV is
\begin{equation}
\mathcal{R}(B) \;:=\; \mathcal{L}\bigl(\widehat{E}_B\bigr) - \mathcal{L}\bigl(E_B^{\star}\bigr) \;\geq\; 0,
\end{equation}
where non-negativity follows from the optimality of $E_B^{\star}$ for $\mathcal{L}$.

\begin{theorem}[Eviction cost regret bound]
\label{thm:regret}
Under the eviction-only abstraction and the unit-size abstraction (Section~\ref{app:smoothness:setup}), for any predictor outputs $\{\widehat{P}_w^{(k)}\}$, the eviction cost regret of PBKV satisfies
\begin{equation}
\label{eq:regret}
\mathcal{R}(B) \;\leq\; \sum_{c\,\in\,\widehat{E}_B \triangle E_B^{\star}} \delta_c
\;\leq\; \frac{1 - \gamma^{K}}{2(1 - \gamma)}\!\!\sum_{c\,\in\,\widehat{E}_B \triangle E_B^{\star}} \!\epsilon_c^{\gamma},
\end{equation}
where $\triangle$ denotes the symmetric set difference and $\delta_c := \bigl|\mathrm{SCORE}(c) - \widehat{\mathrm{SCORE}}(c)\bigr|$. As the prediction error vanishes, $\mathcal{R}(B) \to 0$.
\end{theorem}

\begin{proof}
Partition $\widehat{E}_B \cup E_B^{\star}$ into $\mathcal{A} := \widehat{E}_B \setminus E_B^{\star}$, $\mathcal{B} := E_B^{\star} \setminus \widehat{E}_B$, and $\mathcal{C} := \widehat{E}_B \cap E_B^{\star}$. Both eviction sets have cardinality $B$, hence $|\mathcal{A}| = |\mathcal{B}|$. By construction of $\widehat{E}_B$ as the minimizer of $\sum_{c \in E} \widehat{\mathrm{SCORE}}(c)$ over cardinality-$B$ sets, and after cancelling the common $\mathcal{C}$-terms,
\begin{equation}
\label{eq:opt-ineq}
\sum_{c \in \mathcal{A}} \widehat{\mathrm{SCORE}}(c) \;\leq\; \sum_{c \in \mathcal{B}} \widehat{\mathrm{SCORE}}(c).
\end{equation}
Therefore
\begin{align}
\mathcal{R}(B) &= \sum_{c \in \mathcal{A}} \mathrm{SCORE}(c) - \sum_{c \in \mathcal{B}} \mathrm{SCORE}(c) \\
&\leq \sum_{c \in \mathcal{A}}\!\bigl(\mathrm{SCORE}(c) - \widehat{\mathrm{SCORE}}(c)\bigr) + \sum_{c \in \mathcal{B}}\!\bigl(\widehat{\mathrm{SCORE}}(c) - \mathrm{SCORE}(c)\bigr) \\
&\leq \sum_{c \in \mathcal{A} \cup \mathcal{B}}\!\delta_c \;=\; \sum_{c\,\in\,\widehat{E}_B \triangle E_B^{\star}}\!\delta_c,
\end{align}
where the first inequality follows from \eqref{eq:opt-ineq} and the second from $|\mathrm{SCORE}(c) - \widehat{\mathrm{SCORE}}(c)| \leq \delta_c$ on each side. The remaining inequality of Eq.~\eqref{eq:regret} follows by applying Lemma~\ref{lem:lipschitz} to each $\delta_c$.
\end{proof}

\stitle{Interpretation.}
Theorem~\ref{thm:regret} bounds the regret of PBKV's eviction decisions on the expected discounted miss count, the canonical cost in caching analysis dating back to Belady~\cite{belady1966study}. Two observations are worth emphasizing. First, the multiplier $(1-\gamma^{K})/\bigl(2(1-\gamma)\bigr)$ matches that of Lemma~\ref{lem:lipschitz}, so the score-level and cost-level bounds share a single $K$-independent upper bound $1/\bigl(2(1-\gamma)\bigr)$. Second, only nodes in the symmetric difference $\widehat{E}_B \triangle E_B^{\star}$ contribute to $\mathcal{R}(B)$. A correctly ranked node, relative to the eviction boundary, contributes zero regret regardless of the magnitude of its individual score error. This boundary-localized regret structure indicates that the bound is governed by ranking errors near the eviction frontier rather than by aggregate score errors, consistent with the ranking-stability view of Corollary~\ref{cor:rank}.

\subsection{Summary}

This appendix establishes the smoothness pillar of PBKV's algorithms-with-predictions guarantee along four layers: identification of $\mathrm{SCORE}(c)$ as the expected discounted miss count under the eviction-only abstraction (Lemma~\ref{lem:emc}), per-node score Lipschitz continuity (Lemma~\ref{lem:lipschitz}), pairwise ranking stability (Corollary~\ref{cor:rank}), and cost-level regret on the resulting proxy cost (Theorem~\ref{thm:regret}). All Lipschitz-type results share a single $K$-independent multiplier $1/\bigl(2(1-\gamma)\bigr)$, independent of $|\mathcal{W}_{\mathrm{act}}|$ and the cache size. By analyzing eviction quality under the same eviction-only abstraction adopted in classical caching theory, the analysis treats PBKV's eviction policy as a self-contained component decoupled from any specific storage-hierarchy mechanism. This justifies treating the predictor as a pluggable module in PBKV's design (Section~\ref{sec:eviction}): any future improvement in predictor accuracy directly tightens all bounds, and the bounds themselves remain meaningful in deployments without second-tier storage.

% \newpage
% \input{checklist.tex}

\end{document}